\documentclass{article}

\usepackage{microtype}
\usepackage{graphicx}
\usepackage{subfigure}
\usepackage{booktabs} 
\usepackage{appendix}

\usepackage{hyperref}

\usepackage[accepted]{icml2024}

\usepackage[mathscr]{euscript}

\usepackage{amsmath}
\usepackage{amssymb}
\usepackage{mathtools}
\usepackage{amsthm}

\usepackage[capitalize,noabbrev]{cleveref}

\theoremstyle{plain}
\newtheorem{theorem}{Theorem}[section]

\newtheorem{lemma}[theorem]{Lemma}

\theoremstyle{definition}
\newtheorem{definition}[theorem]{Definition}

\theoremstyle{remark}

\usepackage[textsize=tiny]{todonotes}

\newcommand{\rA}{\textrm{(A)}}
\newcommand{\rB}{\textrm{(B)}}
\newcommand{\rC}{\textrm{(C)}}

\def\reals{{\mathcal R}}

\newcommand{\D}{\mathcal{D}}
\newcommand{\K}{\mathcal{K}}
\newcommand{\R}{\mathcal{R}}

\newcommand{\Bcal}{\mathcal{B}}

\newcommand{\Ncal}{\mathcal{N}}

\newcommand{\ignore}[1]{}

\def\reals{{\mathbb R}}

\def\bold0{\mathbf{0}}

\newcommand{\eps}{\varepsilon}

\newcommand{\A}{\mathcal{A}}
\newcommand{\mO}{\mathcal{O}}

\newcommand{\F}{\mathcal{F}}
\newcommand{\tg}{\tilde{g}}
\newcommand{\bg}{\bar{g}}

\newcommand{\Z}{\mathcal{Z}}

\newcommand{\tsigma}{\tilde{\sigma}}

\newcommand{\tq}{\tilde{q}}

\newcommand{\PS}{\mathcal{PS}}

\newcommand{\bs}{\Bar{s}}
\newcommand{\err}{\textbf{err}}

\newcommand{\supp}[1]{\textbf{Supp}\{#1\}}
\newcommand{\sigmal}{\sigma_L}
\newcommand{\Exp}[1]{\mathbb{E}#1}
\newcommand{\ExpB}[1]{\mathbb{E}\left[#1\right]}
\newcommand{\condExp}[2]{\mathbb{E}_{#2}\left[#1\right]}
\newcommand{\prob}[1]{\mathbb{P}\left\{#1\right\}}
\newcommand{\norm}[1]{\left\|#1\right\|}
\newcommand{\normsq}[1]{\left\|#1\right\|^2}
\newcommand{\dotprod}[2]{\left\langle#1,#2\right\rangle}
\newcommand{\f}[1]{f\!\left(#1\right)}
\newcommand{\df}[1]{\nabla f\!\left(#1\right)}
\newcommand{\fii}[1]{f_i\!\left(#1\right)}
\newcommand{\dfi}[1]{\nabla f_i\!\left(#1\right)}
\newcommand{\diver}[1]{\mathbb{D}_\alpha\left(#1\right)}
\def\al#1\eal{\begin{align}#1\end{align}}
\def\als#1\eals{\begin{align*}#1\end{align*}}
\newcommand{\non}{\nonumber\\}
\renewcommand{\S}{\mathscr{S}}

\newcommand{\sgd}{\text{SGD}}
\newcommand{\alg}{DP-$\mu^2$}

\icmltitlerunning{Efficient DP Strategies for Centralized Systems}

\begin{document}

\twocolumn[
\icmltitle{Private and Federated Stochastic Convex Optimization: \\ Efficient Strategies for Centralized Systems}



\icmlsetsymbol{equal}{*}

\begin{icmlauthorlist}
\icmlauthor{Roie Reshef}{equal,tech}
\icmlauthor{Kfir Yehuda Levy}{equal,tech}
\end{icmlauthorlist}

\icmlaffiliation{tech}{Department of Electrical and Computer Engineering, Technion, Haifa, Israel}

\icmlcorrespondingauthor{Roie Reshef}{sror@campus.technion.ac.il}
\icmlcorrespondingauthor{Kfir Yehuda Levy}{kfiryehud@gmail.com}

\icmlkeywords{Machine Learning, ICML}

\vskip 0.3in
]



\printAffiliationsAndNotice{\icmlEqualContribution} 

\begin{abstract}
This paper addresses the challenge of preserving privacy in Federated Learning (FL) within centralized systems, focusing on both trusted and untrusted server scenarios.
We analyze this setting within the Stochastic Convex Optimization (SCO) framework, and devise methods that ensure Differential Privacy (DP) while maintaining optimal convergence rates for homogeneous and heterogeneous data distributions.
Our approach, based on a recent stochastic optimization technique, offers linear computational complexity, comparable to non-private FL methods, and reduced gradient obfuscation.
This work enhances the practicality of DP in FL, balancing privacy, efficiency, and robustness in a variety of server trust environments.
\end{abstract}

\section{Introduction}
Federated Learning (FL) is a novel framework in Machine Learning (ML), enabling collaborative learning among myriads of decentralized devices or systems \citep{mcmahan2017communication,kairouz2021advances}.
Privacy is a paramount issue in FL, underscoring the critical need to prevent the disclosure of private information as a result of the training process.

To address this challenge, Differential Privacy (DP) has emerged as a robust framework for quantifying and managing privacy risks \citep{odo_dp,calib_dp}.
DP offers a formal guarantee that the outcome of a data analysis does not significantly change when any single individual's data is added or removed, thereby ensuring that individual privacy is maintained \citep{dwork2014algorithmic}.
Implementing DP guarantees in FL has been extensively studied e.g.~in~\citep{huang2019dp,fl_dp,girgis2021shuffled,fl_hetero,private_fed,lowy2023private}, necessitating a nuanced equilibrium between data obfuscation and maintaining model accuracy.
This intricate balance is crucial to preserve the efficacy of the learning process while ensuring privacy.
The critical challenge arising in this context is in ensuring the confidentiality of the model updates transmitted over the distributed network.

FL differentiates between device-based learning, limited by device availability and bandwidth, and silo-based learning, involving full participation from all machines, often across organizations~\citep{kairouz2021advances}.
This paper focuses on silo-based learning, crucial for achieving accurate and consistent models, and we shall refer to silos as "machines".

In this work, we investigate DP guarantees for FL within the Stochastic Convex Optimization (SCO) framework, a cornerstone in the design and analysis of ML algorithms~\citep{shalev2009stochastic,shalev2014understanding}.
And we focus on centralized scenarios where a (central) server employs $M$ machines in parallel.
So far, this setting has been mainly investigated in the context of finite-sum problems (ERM), e.g.~in \citep{huang2019dp,fl_dp,girgis2021shuffled,fl_hetero}.
Nevertheless, translating ERM guarantees to population loss guarantees leads to sub-optimal bounds, see e.g.~\citep{opt_private}.

Conversely, optimal guarantees for the population loss in this context were recently substantiated in~\citep{private_fed}.
This work provides different guarantees, depending on the level of trust that machines have in a central server.
Concretely, for \textbf{(i)} \emph{untrusted} server scenarios, this work substantiates a convergence rate of $O\left(\frac{1}{\sqrt{nM}}+\frac{\sqrt{d}}{\epsilon n\sqrt{M}}\right)$, where $M$ is the number of machines, $n$ is the number of samples used per machine during the training process, $d$ is the dimension of the problem, and $\eps$ is the level of ensured DP.
Conversely, \textbf{(ii)} for the case of a \emph{trusted} server, an improved rate of $O\left(\frac{1}{\sqrt{nM}}+\frac{\sqrt{d}}{\epsilon nM}\right)$ is substantiated.
\citep{private_fed} also provide a matching lower bound for these rates.
These guarantees are substantiated for the \emph{homogeneous} case, where the data distribution is identical across all machines.

Nevertheless, as we detail in~\Cref{sec:other}, the approach of~\citep{private_fed} requires a computational complexity that is proportional to $|\S|^{3/2}$, where $|\S|$ is the size of overall data used by all machines.
This is substantially higher compared to standard non-private FL approaches that require only $\propto|\S|$ computations.
Moreover, the above bounds are only substantiated for the homogeneous case; and slightly worse bounds are substantiated for the more general heterogeneous case.

Our paper explores methods that ensure DP for centralized FL scenarios, under the setting of SCO.
Our contributions:
\begin{itemize}
\item\textbf{Untrusted Server.}
In the case of untrusted server, we provide an approach that ensures an optimal convergence rate of $O\left(\frac{1}{\sqrt{nM}}+\frac{\sqrt{d}}{\epsilon n\sqrt{M}}\right)$, with a level $\epsilon$ of DP.
Our bound applies simultaneously to both homogeneous and heterogeneous cases.
\item\textbf{Trusted Server.}
In the case of trusted server, we provide an approach that ensures an optimal convergence rate of $O\left(\frac{1}{\sqrt{nM}}+\frac{\sqrt{d}}{\epsilon nM}\right)$, with a level $\epsilon$ of DP.
Our bound applies simultaneously to both homogeneous and heterogeneous cases.
\item The computational complexity of our approaches is linear in $|\S|$; matching the complexity standard non-private FL approaches \citep{dekel2012optimal}.
\end{itemize}
Our results build on a recent stochastic optimization technique named $\mu^2$-SGD~\citep{mu_square}, that combines two different momentum mechanisms into SGD.
We customize this approach to the centralized FL setting, and show it is less sensitive to a change of a single sample in $\S$ compared to standard SGD, which allows to obtain DP with substantially less gradient obfuscation.

\textbf{Remark.} Our paper categorizes the FL environment into trusted or untrusted server scenarios.
Our definition of an untrusted server mirrors that in \citep{private_fed}.
However, this latter work introduces a nuanced concept of a trusted shuffler, capable only of shuffling but not aggregating gradient message identities.
We juxtapose this with our trusted server scenario, essentially treating the shuffler and server as a single entity.
It is an interesting open question to understand whether one can obtain the same guarantees that we derive for the trusted server scenario, under the more nuanced assumption of a trusted shuffler.

\paragraph{Related Work}~\\
\vspace{-20pt}
\paragraph{SCO with DP guarantees.} There exists a rich literature in the context of differential privacy (DP) within stochastic convex optimization (SCO).
Foundational work in empirical risk minimization (ERM) within the finite-sum problems has been substantially contributed to by a series of studies, among are \citep{chaudhuri2011differentially,kifer2012private,thakurta2013differentially,song2013stochastic,duchi2013local,ullman2015private,talwar2015nearly,wu2017bolt,wang2017differentially,iyengar2019towards,kairouz2021practical,avella2023differentially,ganesh2023faster}.
Nevertheless, these works primarily focus on training loss (ERM); and translating these results to population loss guarantees via standard uniform convergence~\citep{shalev2009stochastic} leads to sub-optimal bounds as described in~\citep{opt_private,snowball}.

The studies of \citep{tight_bound,opt_private} have advanced our understanding by establishing population loss guarantees for DP-SCO, and the latter has provided optimal bounds; albeit with super-linear computational complexity $\propto|\S|^{3/2}$.
This was later improved by \citep{snowball}, which attained the optimal guarantees with a sample complexity which is linear in $|\S|$, but provides privacy only on the final iterate.

\paragraph{FL with DP guarantees.} We have already mentioned previous works that analyze FL with DP guarantees under the SCO setting.
Another notable work in this context is \citep{cheu2022shuffle}, which obtains the optimal bounds for the trusted server case, albeit using an expensive vector-shuffling routine, that leads to a computational complexity which is even larger than $|\S|^{3/2}$ (see Cor.~\text{B.11} therein).

\section{Preliminaries}
\label{sec:prelim}

In this section, we provide the necessary background for analyzing private learning in the context of the standard (single machine) setting, and will touch upon the necessary background towards the more complex federated (or parallel) learning setting.

\subsection{Convex Loss Minimization}
\label{sec:mini}
We focus on Stochastic Convex Optimization (SCO) scenarios where the objective function, $f:\K\mapsto\reals$, is convex and takes the following form:
\al\label{eq:Objective}
\f{x}=\condExp{\f{x;z}}{z\sim\D}
\eal 
here, $\K\subset\reals^d$ is a compact convex set, and $\D$ represents an unknown data distribution from which we may draw i.i.d.~samples.
A learning algorithm receives a dataset $\S=\{z_1,\ldots,z_n\}\subset\Z^n$ ($\Z$ is the set where the samples reside), of $n$ i.i.d. samples from $\D$, and outputs a solution $x_{\text{output}}\in\K$. 
Our performance measure is the expected excess loss $\R(x_{\text{output}})$, defined as:
\al\label{eq:ExpectedLoss}
\R(x_{\text{output}})=\ExpB{\f{x_{\text{output}}}}-\min_{x\in\K}\{\f{x}\}
\eal
 expectation is taken w.r.t.~the randomness of the samples, as well as w.r.t.~the (possible) randomization of the algorithm.

We focus on first-order optimization methods that iteratively employ the samples in $\S$ to create a sequence of query points, culminating in a solution  $x_{\text{output}}\in\K$.
To elaborate, at each iteration $t$, such iterative methods maintain a query point $x_t\in\K$, calculated from previous query points and samples ${z_1,\ldots,z_{t-1}}$.
The subsequent query point $x_{t+1}$ is then determined using $x_t$ and a gradient estimate $g_t$.
This estimate is obtained by drawing a fresh sample $z_t\sim\D$ (taken from $\S$), independently of  past samples, and calculating $g_t=\df{x_t;z_t}$, where the derivative is with respect to $x$.

The independence of the samples ensures that $g_t$ is an unbiased estimator of $\df{x_t}$, in the sense that $\ExpB{g_t\vert x_t}=\df{x_t}$.
It is useful to conceptualize the calculation of $g_t=\df{x_t;z_t}$ as a sort of (noisy) Gradient Oracle.
Upon receiving a query point $x_t\in\K$, this Oracle outputs a vector $g_t\in\reals^d$, serving as an unbiased estimate of $\df{x_t}$.

\paragraph{Assumptions}
We will make the following assumptions.
\\\textbf{Bounded Diameter:}
There exists $D>0$ such that $\max_{x,y\in\K}\norm{x-y}\leq D$
\\\textbf{Convexity:}
The function $f(\cdot;z)$ is convex $\forall z\in\supp{\D}$.

We  also make these assumption about $\f{\cdot;z},\forall z\in\Z$:
\\\textbf{Lipschitz:}
There exists $G>0$ such that:
\als
|\f{x;z}-\f{y;z}|\leq G\norm{x-y},\quad\forall x,y\in\K
\eals
This also implies that $\norm{\df{x;z}}\leq G, \forall x\in\K$.
\\\textbf{Smoothness:}
There exists $L>0$ such that:
\als
\norm{\df{x;z}-\df{y;z}}\leq L\norm{x-y},\quad\forall x,y\in\K
\eals
Since these assumptions hold for $\f{x;z}$ for every $z\in\supp{\D}\subset\Z$, they also hold for $\f{x}$.
The above Lipschitz and Smoothness assumptions imply the following:
\\\textbf{Bounded Variance:} There exist $0\leq\sigma\leq G$ such that:
\al\label{def:BoundedVar}
\Exp{\normsq{\df{x;z}-\df{x}}}\leq\sigma^2,\quad\forall x\in\K
\eal
\textbf{Bounded Smoothness Variance:} $\exists\sigma_L\in[0, L]$ such,
\al\label{def:BoundedSmoothnessVar}
\Exp{\normsq{(\df{x;z}-\df{x})-(\df{y;z}-\df{y})}} \non
\leq\sigmal^2\normsq{x-y},\quad\forall x,y\in\K
\eal
For completeness, we provide a proof in \Cref{proof:asmp}.
\\\textbf{Notations:} We will employ the following notation, Projection $\Pi_\K(x)=\arg\min_{y\in\K}\norm{x-y}$, the notation $[n]=\{1,\ldots,n\}$, the notation $\alpha_{1:t}=\sum_{\tau=1}^t\alpha_\tau$.

\subsection{Differential Privacy}
\label{sec:dp}
In order to discuss privacy, we need to measure how does private information affects the outputs of an algorithm.
The idea behind Differential Privacy (DP) is to look at the output of the algorithm with the private information and without it, and measure the difference between these outputs.
The more similar they are, the more private the algorithm is.
Concretely, similarity is measured w.r.t.~the difference between the probability distributions of the (randomized) outputs.

\paragraph{R{\'{e}}nyi Divergence}
Below is the definition of R{\'{e}}nyi divergence, a popular difference measure between probability distributions, which is prevalent in the context of DP.
\begin{definition}[R{\'{e}}nyi Divergence~\citep{renyi}]
Let $1<\alpha<\infty$, and let $P,Q$ be probability distributions such that $\supp{P}\subseteq\supp{Q}$.
The R{\'{e}}nyi divergence of order $\alpha$ between $P$ and $Q$ is defined as:
\als
\diver{P\|Q}=\frac{1}{\alpha-1}\log\left(\condExp{\left(\frac{P(X)}{Q(X)}\right)^{\alpha-1}}{X\sim P}\right)
\eals
We follow with the convention that $\frac{0}{0}=0$.
If $\supp{P}\nsubseteq\supp{Q}$, then the R{\'{e}}nyi divergence is defined to be $\infty$.
Divergence of orders $\alpha=1,\infty$ are defined by continuity.
\end{definition}
{\flushleft\textbf{Notation:}}
If $X\sim P, Y\sim Q$, we will use $\diver{P\|Q}~\&~\diver{X\|Y}$ interchangeably.

Adding Gaussian noise is a popular primitive in DP, and we will therefore find the following lemma useful:
\begin{lemma}\label{lem:div_gauss}
Let $P\sim\Ncal(\mu,I\sigma^2)$ and $Q\sim\Ncal(\mu+\Delta,I\sigma^2)$, two Gaussian distributions. 
Then, $\diver{P\|Q}=\frac{\alpha\|\Delta\|^2}{2\sigma^2}$.
\end{lemma}
For completeness we provide a proof in \Cref{proof:div_gauss}.

\paragraph{Differential Privacy Definitions}
\begin{definition}[Differential Privacy~\citep{odo_dp,calib_dp}]
A randomized algorithm $\A$ is $(\epsilon,\delta)$-diferentially private, or $(\epsilon,\delta)$-DP, if for all neighbouring datasets $\S,\S'$ that differs in a single element, and for all events $\mO$, we have:
\als
\prob{\A(\S)=\mO}\leq e^\epsilon\prob{\A(\S')=\mO}+\delta
\eals
\end{definition}
The term \emph{neighbouring datasets} refers to $\S,\S'$ being \emph{ordered} sets of data points that only differ on \emph{a single element}.

In this paper, we would adopt a different and prevalent  privacy measure that relies on the R{\'{e}}nyi divergence:
\begin{definition}[R{\'{e}}nyi Differential Privacy~\citep{rdp}]\label{def:RDP}
For $1\leq\alpha\leq\infty$ and $\epsilon\geq0$, a randomized algorithm $\A$ is $(\alpha,\epsilon)$-R{\'{e}}nyi diferentially private, or $(\alpha,\epsilon)$-RDP, if for all neighbouring datasets $\S,\S'$ that differ in a single element,
\als
\diver{\A(\S)\|\A(\S')}\leq\epsilon
\eals
\end{definition}

Curiously, it is known that RDP implies DP.
This is established in the next lemma:
\begin{lemma}[\citep{rdp}]\label{lem:dp_rdp}
If $\A$ satisfies $(\alpha,\epsilon)$-RDP, then for all $\delta\in(0,1)$, it also satisfies $\left(\epsilon+\frac{\log(1/\delta)}{\alpha-1},\delta\right)$-DP.
In particular, if $\A$ satisfies $\left(\alpha,\frac{\alpha\rho^2}{2}\right)$-RDP for every $\alpha\geq1$, then for all $\delta\in(0,1)$, it also satisfies $\left(\frac{\rho^2}{2}+\rho\sqrt{2\log(1/\delta)},\delta\right)$-DP.
\end{lemma}
For completeness we provide a proof in \Cref{proof:dp_rdp}.

\subsection{Federated Learning}
\label{sec:fl}
In federated learning (FL), multiple machines collaborate together to solve a joint problem.
We assume that there exist $M$ machines, and that each machine $i\in[M]$ may independently draw i.i.d.~samples from a distribution $\D_i$.
Similarly to the standard SCO setting, we can associate an expected convex loss $\fii{\cdot}$ with machine $i\in[M]$, which is defined as $\fii{x}:=\condExp{\fii{x;z^i}}{z^i\sim\D_i}$.
We denote the joint objective of all machines by $f:\K\mapsto\reals$, where:
\als
\f{x}:=\frac{1}{M}\sum_{i\in[M]}\fii{x}:=\frac{1}{M}\sum_{i\in[M]}\condExp{\fii{x;z^i}}{z^i\sim\D_i}
\eals
Thus, the objective is an average of $M$ functions $\{f_i:\K\mapsto\reals\}_{i\in[M]}$, and  each such $\fii{\cdot}$ can be written as an expectation over losses $\fii{\cdot,z^i}$ where the $z^i$ are drawn from some distribution $\D_i$ which is unknown to the learner.
For ease of notation, in what follows we will not explicitly denote $\condExp{\cdot}{z^i\sim\D_i}$ but rather use $\ExpB{\cdot}$ to denote the expectation w.r.t.~all randomization.

In order to collaboratively minimize $\f{\cdot}$, the machines may synchronize and communicate through a central machine called the \emph{Parameter Server} ($\PS$).
We will focus on the most common parallelization scheme~\citep{dekel2012optimal}, where at every round $t$ the $\PS$ communicates a query point $x_t$ to all machines.
Then, every machine performs a gradient computation based on its local data and communicates gradient estimates back to the $\PS$, which in turn aggregates these estimates and updates the query point.
Similarly to the standard SCO setting, our performance measure is the Expected loss w.r.t.~$\f{\cdot}$ (see~\Cref{eq:ExpectedLoss}).
Moreover, we shall assume that each machine $i\in[M]$ maintains and utilizes a dataset $\S_i$ of samples that are drawn i.i.d.~from $\D_i$.
Finally, we assume that for every $i\in[M]$, the functions $\{\fii{\cdot,z}\}_{z\in\Z_i}$ are convex $G$-Lipschits and $L$-smooth, as well as make bounded variance and smoothness variance assumptions (see~\Cref{def:BoundedVar,def:BoundedSmoothnessVar}).

In the context of privacy, we explore the case where the FL process should ensure DP guarantees individually for every machine $i\in[M]$.
The challenge in this context arises due to the fact that each machine communicates \emph{several times} with the $\PS$ which may uncover more information regarding its private dataset.
Concretely, we assume that machines do not trust each other, and thus cannot allow other machines to uncover private data.
We deal with two cases:
\textbf{(i) Untrusted Server.}
Here the $\PS$ may not uncover private data, and we are therefore required to ensure DP-guarantees w.r.t.~information that the machines communicate to the $\PS$ (i.e.~gradient estimates).
\textbf{(ii) Trusted Server.}
In this case all machines trust the $\PS$ and are therefore allowed to send information that may expose their privacy.
Nevertheless, we still require that machines may not uncover private data of one another from the information that is received from the $\PS$.
Thus, for each machine $i\in[M]$, we are required to ensure DP-guarantees w.r.t.~information that is distributed by the $\PS$ (i.e.~query points).

\section{Our Algorithm Mechanisms}
\label{sec:mech}
Our approach builds on the standard gradient descent template, that we combine with  two additional mechanisms, both are crucial to the result that we obtain.
Next we elaborate on these mechanisms, and in the next sections discuss and analyze their combination.

\subsection{$\mu^2$-SGD}
The $\mu^2$-SGD \citep{mu_square} is a variant of standard SGD with several modifications.
Its update rule is of the following form:
$w_1=x_1\in\K$, and $\forall t>1$:
\al\label{eq:MU2SGD} 
w_{t+1}=&\Pi_\K\left(w_t-\eta\alpha_t d_t\right) \non
x_{t+1}=&\frac{\alpha_{1:t}}{\alpha_{1:t+1}}x_t+\frac{\alpha_{t+1}}{\alpha_{1:t+1}}w_{t+1}
\eal
where $\{\alpha_t>0\}_t$ are importance weights that may unequally emphasize different update steps.
Concretely we will employ $\alpha_t\propto t$, which puts more emphasis on the more recent updates.
Moreover, the $\{x_t\}_t$'s are a sequence of weighted averages of the iterates $\{w_t\}_t$, and $d_t$ is an estimate for the gradient at the average point, i.e.~of $\df{x_t}$.
This is different than standard SGD which employs estimates for the gradients at the iterates, i.e.~of $\df{w_t}$.
This approach is related to a technique called Anytime-GD~\citep{anytime}, which is strongly-connected to the notions of momentum and acceleration~\citep{anytime, kavis2019unixgrad}.

While in the natural SGD version of Anytime-GD, one would employ the estimate $\df{x_t;z_t}$, the $\mu^2$-SGD approach suggests to employ a variance reduction mechanism to yield a \emph{corrected momentum} estimate $d_t$ in the spirit of \citep{storm}.
This is done as follows: $d_1:=\df{x_1;z_1}$, and $\forall t>2$:
\al\label{eq:STORM}
d_t=\df{x_t;z_t}+(1-\beta_t)(d_{t-1}-\df{x_{t-1};z_t})
\eal
where $\beta_t\in[0,1]$ are called \emph{corrected momentum} weights.
It can be shown by induction that $\ExpB{d_t}=\ExpB{\df{x_t}}$, nevertheless in general $\ExpB{d_t\vert x_t}\neq\df{x_t}$ (in contrast to standard SGD estimators).
Nevertheless, it was shown in \citep{mu_square} that upon choosing \emph{corrected momentum} weights of $\beta_t:=1/t$, the above estimates enjoy an error reduction, i.e.~$\Exp{\normsq{\eps_t}}:=\Exp{\normsq{d_t-\df{x_t}}}\leq O(\tsigma^2/t)$ at step $t$, where $\tsigma^2\leq O(\sigma^2+\sigmal^2D^2)$.
Implying that the error decreases with $t$, contrasting with standard SGD where the variance $\Exp{\normsq{\eps^{\textnormal{SGD}}_t}}:= \Exp{\normsq{g_t-\df{x_t}}}$ remains uniformly bounded by $\sigma^2$.

\subsection{Noisy-$\mu^2$-SGD}
\label{sec:NoisyMU2}
An additional mechanism that we utilize consists of adding zero mean noise to the gradients, which in turn adds privacy to the learning algorithm~\citep{dl_dp}.
Combining this idea with the $\mu^2$-SGD approach (\Cref{eq:MU2SGD}) induces the following update rule:
\al\label{eq:NoisySGD+mu2}
w_{t+1}&=\Pi_\K\left(w_t-\eta(\alpha_t d_t+Y_t)\right) \non
x_{t+1}&=\frac{\alpha_{1:t}}{\alpha_{1:t+1}}x_t+\frac{\alpha_{t+1}}{\alpha_{1:t+1}}w_{t+1}
\eal
where $d_t$ is a corrected momentum estimate (\Cref{eq:STORM}), and $\{Y_t\sim P_t\}$ is a sequence of independent zero mean noise terms. 
Curiously, there is a natural tradeoff in choosing the noise magnitude: larger noise degrades the convergence, but improves privacy.

\section{Noisy-$\mu^2$-SGD for Differentially-Private FL}
\label{sec:mult}

In this section we deal with the case of multiple machines, as described in~\Cref{sec:fl}.
Our parallelization approach, appearing in~\Cref{alg:untrust,alg:trust} is based on the well known ``minibacth-SGD" (a.k.a.~``parallel-SGD") algorithmic template~\citep{dekel2012optimal}.
In this template the server sends a query point to all machines.
Then, every machine $i\in[M]$ computes a gradient estimate based on its local data, and communicate it back to the server.
And the latter averages these gradient estimates, and updates the models using a gradient step.

\paragraph{Algorithmic Approach}
Our approach is depicted in~\Cref{alg:untrust} (untrusted server), and in~\Cref{alg:trust} (trusted server).
It can be seen that our approach is inspired by Noisy-\alg (\Cref{sec:NoisyMU2}), and therefore differs than the standard minibacth-SGD in three different aspects:
\textbf{(i)} the parameter server queries the gradients at the averages $\{x_t\}_{t\in[T]}$, which is in the spirit of Anytime-GD;
\textbf{(ii)} Each machine $i\in[M]$ maintains and updates a (weighted) corrected momentum estimate $q_{t,i}:=\alpha_td_{t,i}$, which is an unbiased estimate of $\alpha_t\dfi{x_t}$, in the spirit of $\mu^2$-SGD (\Cref{eq:MU2SGD,eq:STORM});
\textbf{(iii)} A synthetic noise is injected, either to the individual estimates $q_{t,i}$ (in case of untrusted server, \Cref{alg:untrust}),or to the aggregated estimate $q_t:=\frac{1}{M}\sum_{i\in[M]}q_{t,i}$ (in case of trusted server, \Cref{alg:trust}).
\\\textbf{Remark:} Note that for the sake of establishing convergence guarantees we need to assume that the samples in $\S_i$ are drawn i.i.d.~from $\D_i$.
Nevertheless, our privacy guarantees do not necessitate this requirement, and for the sake of DP guarantees we only assume Lipschitz continuity and smoothness of $\fii{\cdot;z^i}$ for all $i\in[M], z^i\in\Z_i$.

\paragraph{Computational Complexity}
From the description of \Cref{alg:untrust,alg:trust}, it can be directly seen that at every round $t$, each machine $i\in[M]$ employs a single sample $z_{t,i}$ which is used to compute two noisy gradients estimates $g_{t,i},\tg_{t-1,i}$. Thus, the overall computational complexity of our approach is linear in the dataset size $|\S|$, where $\S:=\cup_{i\in[M]}\S_i$.

Next we provide a general sensitivity and error analysis that applies to both \Cref{alg:untrust,alg:trust}.
Then we establish privacy  and convergence guarantees for each setting.

\subsection{Gradient Error \& Sensitivity Analysis}
Here we present and discuss an analysis for the error of the $q_{t,i}$ estimates, as well as analyze their sensitivity to a single point in the dataset.
This analysis applies to both \Cref{alg:untrust,alg:trust}, and will later serve us in deriving  privacy and convergence guarantees.

Prior to the analysis, we shall introduce some notation.
We will define $g_{t,i}:=\df{x_t;z_{t,i}}, \tg_{t,i}:=\df{x_t,z_{t+1,i}}, \bg_{t,i}:=\df{x_{t,i}}$.
These notations allow us to write the update rule for $d_{t,i}$ in~\Cref{alg:untrust} (as well as~\Cref{alg:trust} ) as follows:
\als
d_{t+1,i}=&g_{t+1,i}+(1-\beta_{t+1})(d_{t,i}-\tg_{t,i}) \\
=&\beta_{t+1}g_{t+1,i}+(1-\beta_{t+1})(d_{t,i}+g_{t+1,i}-\tg_{t,i}) 
\eals
Using the notation $q_{t+1,i}:=\alpha_{t+1}d_{t+1,i}$, we can rewrite the above equation:
\als
q_{t+1,i}:=&\alpha_{t+1}d_{t+1,i}=\alpha_{t+1}\beta_{t+1}g_{t+1,i} \\
+&(1-\beta_{t+1})\alpha_{t+1}((q_{t,i}/\alpha_t)+g_{t+1,i}-\tg_{t,i}) 
\eals
Since in~\Cref{alg:untrust}~(as well as~\Cref{alg:trust}) we choose $\alpha_t=t$ and $\beta_t=1/\alpha_t$, we have $(1-\beta_{t+1})\alpha_{t+1}=\alpha_t$, therefore:
\al\label{eq:q_tUpdate1}
q_{t+1,i}=q_t+g_{t+1,i}+\alpha_t(g_{t+1,i}-\tg_{t,i}) 
\eal
Now, let us define two sequences $\{s_{t,i}\},\{\bs_{t,i}\}$: $s_{1,i}:=g_{1,i}$, and $\bs_{1,i}:=\bg_{1,i}$, and $\forall t>1$,
\al\label{Eq:S_tBarS_t}
s_{t,i}:=g_{t,i}+\alpha_{t-1}(g_{t,i}-\tg_{t-1,i}) \non
\bs_{t,i}:=\bg_{t,i}+\alpha_{t-1}(\bg_{t,i}-\bg_{t-1,i})
\eal
Note that the definitions of $s_{t,i},\bs_{t,i}$ and $g_{t,i},\tg_{t,i},\bg_{t,i}$ imply:
\als
\condExp{s_{t,i}-\bs_{t,i}}{t-1}=\condExp{g_{t,i}-\bg_{t,i}}{t-1}=0
\eals
where $\condExp{\cdot}{t-1}$ denotes conditional expectation over all randomization prior to time $t$, i.e.~$\condExp{\cdot}{t-1}:=\ExpB{\cdot\vert\{z_{\tau,i}\}_{\tau\in[t-1],i\in[M]},\{Y_{\tau,i}\}_{\tau\in[t-1],i\in[M]}}$.
Thus, that above implies that the sequence $\{s_{t,i}-\bs_{t,i}\}_t$ is a Martingale difference sequence w.r.t.~the natural filtration induced by the data-samples and injected noises.

The next lemma shows that for our specific choices of $\{\alpha_t,\beta_t\}$, we can represent $q_{t,i}$ as a sum of $\{s_{\tau,i}\}_{\tau=1}^t$.
Importantly, each such $s_{\tau,i}$ terms is related to an individual data sample $z_{\tau,i}$.
Moreover, it shows that we can represent the error of the weighted gradient-estimate $\eps_{t,i}:=q_{t,i}-\alpha_t\dfi{x_t}$ as sum of Martingale difference sequence $\{s_{\tau,i}-\bs_{\tau,i}\}_{\tau=1}^t$.
\begin{lemma}\label{lem:Q_Srelation}
The choices that we make in~\Cref{alg:untrust,alg:trust}, i.e.~$\alpha_t = t$, and $\beta_t=1/\alpha_t$; imply:
\als
q_{t,i}=\sum_{\tau=1}^t s_{\tau,i}\quad\&\quad\eps_{t,i}=\sum_{\tau=1}^t(s_{\tau,i}-\bs_{\tau,i})
\eals
\end{lemma}
We prove of the above lemma in~\Cref{proof:Q_Srelation}.

Our next lemma provides bound on the increments $s_{t,i}$:
\begin{lemma}\label{lem:bound_s}
Let $\K\subset\reals^d$ be a convex set of diameter $D$, and $\{\fii{\cdot;z^i}\}_{z^i\in\Z_i}$ be a family of convex $G$-Lipschitz and $L$-smooth functions.
Also define $S:=G+2LD, \tsigma=\sigma+2\sigmal D$, then:
\als
\norm{s_{t,i}}\leq S\quad\&\quad\Exp{\normsq{s_{t,i}-\bs_{t,i}}}\leq\tsigma^2
\eals
Where the in expectation bound further assumes that the samples in $\S_i$ arrive from i.i.d.~$\D_i$.
\end{lemma}
We prove the above lemma in~\Cref{proof:bound_s}.

\Cref{lem:Q_Srelation,lem:bound_s} enable us to bound $\eps_{t,i}$:
\begin{lemma}\label{lem:bound_eps}
Our algorithms ensure: $\Exp{\normsq{\eps_{t,i}}}\leq\tsigma^2t$.
\end{lemma}
We prove the above lemma in~\Cref{proof:bound_eps}.

\paragraph{Comparison to SGD}
Here we will make an informal discussion of the guarantees of the above lemma  in comparison to its parallel noise-injected SGD version that appears.
Concretely, in SGD the contribution  of a single sample $z_{t,i}$ to the gradient estimate is encapsulated in $\dfi{w_t;z_{t,i}}$, and the magnitude of this contribution can be of order $O(G)$ (as per our bounded gradient assumption).
Moreover, in SGD the expected error of the gradient estimate (without the injected noise $Y_{t,i}$) is 
$\Exp{\normsq{\err_t^{\sgd}}}=\Exp{\normsq{\dfi{w_t;z_{t,i}}-\dfi{w_t}}}\leq\sigma^2$.
Now, to make a proper comparison with our algorithm, we will compare these bounds to our (unweighted) gradient estimate, namely $d_{t,i}$.
Recalling that $d_{t,i}=q_{t,i}/\alpha_t$ and $\alpha_t=t$, \Cref{lem:Q_Srelation,lem:bound_s}, imply that the contribution of a single sample $z_{\tau,i}$ to the gradient estimate $d_{t,i}$ is encapsulated in $s_{\tau,i}/\alpha_t$ (assuming $\tau\leq t$), and the magnitude of this contribution can be of order $O(S/t)$.
Moreover, $\Exp{\normsq{\err_t}}=\Exp{\normsq{d_{t,i}-\dfi{x_t}}}\leq\tsigma^2t/\alpha_t^2=\tsigma^2/t$.
Thus, in comparison to SGD, the errors of our gradient estimates decay with $t$, and a single sample $z_{\tau,i}$ directly affects all estimates $\{d_{t,i}\}_{t\geq \tau}$, but its affect decays with $t$.
Note that the property that $\Exp{\normsq{d_{t,i}-\dfi{x_t}}}=\tsigma^2/t$ was already demonstrated in~\citep{mu_square}, and we provide its proof for completeness.

In the next sections we will see how~\Cref{lem:bound_s} plays a key role in deriving privacy guarantees for our \alg approach appearing in~\Cref{alg:untrust,alg:trust}.

\section{Untrusted Server}
\label{sec:untrust}

In the untrusted server case, the server is not allowed to access private information, thus each machine must apply a privacy mechanism before sending the gradient estimate to the server.
Thus, each machine will have its own $q_{t,i}$, but send $\tq_{t,i}:=q_{t,i}+Y_{t,i}$ instead.
Since the server receives protected information, the server does not need to apply any privacy mechanism before sending the query point $x_t$ to the machines.
\Cref{alg:untrust} depicts our approach for the untrusted server case.

\begin{algorithm}[t]
\begin{algorithmic}
\caption{\alg-FL for Untrusted Server}\label{alg:untrust}
\STATE \textbf{Inputs:} \#iterations $T$, \#machines $M$, initial point $x_1$, learning rate $\eta>0$, importance weights $\{\alpha_t=t\}$, corrected momentum weights $\{\beta_t=1/\alpha_t\}$, noise distributions $\left\{P_{t,i}=\Ncal(0,I\sigma_{t,i}^2)\right\}$, per-machine $i\in[M]$ a dataset of samples $\S_i=\{z_{1,i},\ldots,z_{T,i}\}$ 
\STATE \textbf{Initialize:} set $w_1=x_1$, and $x_0=x_1$
\FOR{every Machine~$i\in[M]$}
\STATE set $d_{0,i}=0$
\ENDFOR
\FOR{$t=1,\ldots,T$}
\FOR{every Machine~$i\in[M]$}
\STATE \textbf{Actions of Machine $i$:}
\STATE Retrieve $z_{t,i}$ from $\S_i$, compute $g_{t,i}=\df{x_t;z_{t,i}}$, and $\tg_{t-1,i}=\df{x_{t-1};z_{t,i}}$
\STATE Update $d_{t,i}=g_{t,i}+(1-\beta_t)(d_{t-1,i}-\tg_{t-1,i})$ and $q_{t,i}=\alpha_td_{t,i}$
\STATE Draw $Y_{t,i}\sim\Ncal\left(0,I\sigma_{t,i}^2\right)$
\STATE Update $\tq_{t,i}= q_{t,i}+Y_{t,i}$
\COMMENT{each machine sends private information}
\ENDFOR
\STATE \textbf{Actions of Server:}
\STATE Aggregate $\tq_t=\frac{1}{M}\sum_{i=1}^M\tq_{t,i}$
\COMMENT{average after adding the noise}
\STATE Update $w_{t+1}=\Pi_\K(w_t-\eta\tq_t)$
\STATE Update $x_{t+1}=(1-\frac{\alpha_{t+1}}{\alpha_{1:t+1}})x_t
+\frac{\alpha_{t+1}}{\alpha_{1:t+1}}w_{t+1}$
\ENDFOR
\STATE \textbf{Output:} $x_T$
\end{algorithmic}
\end{algorithm}

\subsection{Privacy Guarantees}
Here we  establish the privacy guarantees of \Cref{alg:untrust}.
Concretely, the following theorem shows how does the privacy of our algorithm depends on the variances of injected noise $\{Y_{t,i}\}_t$.
\begin{theorem}[Privacy Guarantees for \Cref{alg:untrust}] 
\label{Thm:Privacy}
Let $\K\subset\reals^d$ be a convex set of diameter $D$, and $\{\fii{\cdot;z^i}\}_{z^i\in\Z_i}$ be a family of convex $G$-Lipschitz and $L$-smooth functions.
Then invoking~\Cref{alg:untrust} with  noise distributions $Y_{t,i}\sim P_{t,i}=\Ncal\left(0,I\sigma_{t,i}^2\right)$, and any learning rate $\eta>0$, ensures that for any machine $i\in[M]$, the resulting sequences $\{\tq_{t,i}\}_{t\in[T]}$ is $\left(\alpha,\frac{\alpha\rho_i^2}{2}\right)$-RDP for any $\alpha>1$, where:
$\rho_i=2S\sqrt{\sum_{\tau=1}^T\frac{1}{\sigma_{\tau,i}^2}}$.
\end{theorem}
The full proof is in \Cref{proof:Privacy}.
\begin{proof}[Proof Sketch]
First assume that $\S_i$ and $\S_i'$ are neighboring datasets, meaning that there exists only a single index $\tau^*\in[T]$ where they differ, i.e.~that $z_{\tau^*}\neq z_{\tau^*}'$.

Now, recalling that we obtain $\tq_{t,i}$ by adding a Gaussian noise $Y_{t,i}$ to $q_{t,i}$, we may use~\Cref{lem:div_gauss} and obtain:
\al\label{eq:RDP_SingleIterateImplicit}
\diver{\tq_{t,i}(\S_i)\|\tq_{t,i}(\S_i')}=\frac{\alpha\Delta_{t,i}^2}{2\sigma_{t,i}^2}
\eal
Using that fact that $q_{t,i}=\sum_{\tau=1}^ts_{\tau,i}$ together with the bound $\norm{s_{\tau,i}}\leq S$, which holds for any $\tau\in[T]$ due to~\Cref{lem:bound_s}, we show that $\Delta_{t,i}\leq2S\cdot\mathbb{I}\{t\geq\tau^*\}$, with $\tau^*$ being the time step where $z_{\tau^*}\neq z_{\tau^*}'$, and $S:=G+2LD$.
We thus get:
\al\label{eq:RDP_SingleStep}
\diver{\tq_{t,i}(\S_i)\|\tq_{t,i}(\S_i')}\leq\frac{2\alpha S^2}{\sigma_{t,i}^2}\cdot\mathbb{I}\{t\geq\tau^*\}
\eal
Using the above together with \Cref{lem:rdp_copm}, gives:
\al\label{eq:Q_Seq_Privacy_untrust}
\diver{\{\tq_{\tau,i}(\S_i)\}_{\tau=1}^t\|\{\tq_{\tau,i}(\S_i')\}_{\tau=1}^t}\leq2\alpha S^2\sum_{\tau=1}^t\frac{1}{\sigma_{\tau,i}^2}
\eal
Using the definition of RDP concludes the proof.
\end{proof}

\subsection{Convergence Guarantees}
Next we state our main theorem for this setting.
\begin{theorem}\label{thm:untrust}
Let $\K\subset\reals^d$ be a convex set of diameter $D$ and $\{f_i(\cdot;z^i)\}_{i\in[M],z^i\in\Z_i}$ be a family of $G$-Lipschitz and $L$-smooth functions over $\K$, with $\sigma\in[0,G], \sigmal\in[0,L]$ as defined in~\Cref{def:BoundedVar,def:BoundedSmoothnessVar}, define $\f{x;z}=\frac{1}{M}\sum_{i=1}^M\fii{x;z^i}$ and $G^*:=\df{x^*}$, where $x^*=\arg\min_{x\in\K}\f{x}$, and $S:=G+2LD, \tsigma:=\sigma+2\sigmal D$, moreover let $T\in\mathbb{N}, \rho>0$.

Then upon invoking \Cref{alg:untrust} with $\alpha_t=t, \beta_t=1/\alpha_t, \eta=\min\left\{\frac{\rho D\sqrt{M}}{2ST\sqrt{d}},\frac{1}{4LT}\right\}$, and $\sigma_{t,i}^2=4S^2T/\rho^2$, and any starting point $x_1\in\K$ and datasets $\{\S_i\in\Z_i^T\}_{i\in[M]}$, then for all $\alpha\geq1$ \Cref{alg:untrust} satisfies $\left(\alpha,\frac{\alpha\rho^2}{2}\right)$-RDP w.r.t~gradient estimate sequences that each machine produces, i.e.~$\{\tq_{t,i}\}_{t\in[T],i\in[M]}$.

Furthermore, if  $\S_i$ consists of i.i.d.~samples from a distribution $\D_i$ for all $i\in[M]$, then \Cref{alg:untrust} guarantees:
\als
\R(x_T^{untrust}):=\ExpB{\f{x_T^{untrust}}}-\min_{x\in\K}\f{x} \\
\leq4D\left(\frac{G^*+2LD}{T}+\frac{2S\sqrt{d}}{\rho T\sqrt{M}}+\frac{\tsigma}{\sqrt{TM}}\right)
\eals
\end{theorem}
The full proof is in \Cref{proof:untrust}.
Notably, the above bounds are optimal.
Moreover, similarly to the standard minibatch-SGD analysis \citep{dekel2012optimal}, these bound do not depend on the level of heterogeneity between machines.
Which is due to the full synchronicity and symmetry between machines, implying that the aggregated $q_t$ satisfies $\ExpB{q_t}=\alpha_t\ExpB{\df{x_t}}$, as well as $\Exp{\normsq{q_t-\alpha_t\df{x_t}}}\leq \tsigma^2t/M$.
\begin{proof}[Proof Sketch]
The privacy guarantees follow directly from \Cref{Thm:Privacy}, and our choice of $\sigma_{t,i}^2$.
Regrading  convergence: in the spirit of $\mu^2$-SGD analysis \citep{mu_square}, we show that:
\al\label{eq:FinalBoundAlmostExplicitUntrusted}
\alpha_{1:T}\R(x_T)\leq\frac{D^2}{\eta}+\eta\sum_{\tau=1}^T\Exp{\normsq{Y_\tau}} \non
+2\alpha_T DG^*+2D\sum_{\tau=1}^T\sqrt{\Exp{\normsq{\eps_\tau}}}
\eal
where we denoted $\eps_\tau: = \frac{1}{M}\sum_{i\in[M]}\eps_{\tau,i}$, and $Y_\tau: = \frac{1}{M}\sum_{i\in[M]}Y_{\tau,i}$. Then, by utilizing Lemma~\ref{lem:bound_eps}, we show that  $\Exp{\normsq{\eps_\tau}}\leq \frac{1}{M}\Exp{\normsq{\eps_{\tau,i}}}\leq \tsigma^2 \tau/M$; and similarly that $\Exp{\normsq{Y_\tau}}\leq \frac{1}{M}\Exp{\normsq{Y_{\tau,i}}}\leq \sigma_\tau^2/M$. Plugging these into the bound above and using $\alpha_t=t$, and our choices of $\eta$ and $\sigma_t$ implies the bound.
\end{proof}

\subsection{Experiments}
We ran \alg on MNIST using a logistic regression model in the untrusted server case.
The parameters are $G=\sqrt{2\cdot785}=39.6, L=785/2=392.5, D=0.1$, which brings us $S=118.1$.
Our model has $d=10\cdot785=7850$ parameters.
We kept $M\cdot T=60,000$, and checked $M=1,10,100$ and $\rho=4,8,16$. 
We show our results in \Cref{tab:untrust}.

\begin{table}[t]
\small
\begin{center}
\caption{\alg on MNIST with untrusted server}
\begin{tabular}{c@{\hspace{0.3cm}}c@{\hspace{0.1cm}}c@{\hspace{0.3cm}}c@{\hspace{0.1cm}}c@{\hspace{0.3cm}}c@{\hspace{0.1cm}}c}
\toprule
 & \multicolumn{2}{c}{M=1, T=60,000} & \multicolumn{2}{c}{M=10, T=6,000} & \multicolumn{2}{c}{M=100, T=600} \\
$\rho$ & Loss & Accuracy & Loss & Accuracy & Loss & Accuracy \\
\midrule
 4 & 2.256 & 69.9\% & 2.267 & 69.4\% & 2.285 & 65.4\% \\
\midrule
 8 & 2.253 & 70.2\% & 2.259 & 70.0\% & 2.274 & 69.8\% \\
\midrule
16 & 2.252 & 70.4\% & 2.255 & 70.1\% & 2.264 & 70.0\% \\
\bottomrule
\end{tabular}
    \label{tab:untrust}
    \end{center}
\end{table}

We can see that by increasing $\rho$, we improve the loss and accuracy, as higher $\rho$ means less privacy.
By increasing $M$ the loss and accuracy become worse, as the data is split between different machine and cannot be used optimally due to not trusting the server.

\section{Trusted Server}
\label{sec:trust}

In the trusted server case, the server is allowed to access private information, therefore machines do not apply a privacy mechanism before sending the gradient estimate to the server.
And each machine maintains its own estimate $q_{t,i}$ (an estimate of $\alpha_t\dfi{x_t}$), and sends it without adding noise.
Since the server receives private information, it must apply a privacy mechanism before sending the query point $x_t$ to the machines.
\Cref{alg:trust} depicts our approach  in the trusted server case.

\begin{algorithm}[t]
\begin{algorithmic}
\caption{\alg-FL for Trusted Server}\label{alg:trust}
\STATE \textbf{Inputs:} \#iterations $T$, \#machines $M$, initial point $x_1$, learning rate $\eta>0$, importance weights $\{\alpha_t=t\}$, corrected momentum weights $\{\beta_t=1/\alpha_t\}$, noise distributions $\left\{P_t=\Ncal(0,I\sigma_t^2)\right\}$, per-machine $i\in[M]$ a dataset of samples $\S_i=\{z_{1,i},\ldots,z_{T,i}\}$ 
\STATE \textbf{Initialize:} set $w_1=x_1$, and $x_0=x_1$
\FOR{every Machine~$i\in[M]$}
\STATE set $d_{0,i}=0$
\ENDFOR
\FOR{$t=1,\ldots,T$}
\FOR{every Machine~$i\in[M]$}
\STATE \textbf{Actions of Machine $i$:}
\STATE Retrieve $z_{t,i}$ from $\S_i$, compute $g_{t,i}=\df{x_t;z_{t,i}}$, and $\tg_{t-1,i}=\df{x_{t-1};z_{t,i}}$
\STATE Update $d_{t,i}=g_{t,i}+(1-\beta_t)(d_{t-1,i}-\tg_{t-1,i})$ and $q_{t,i}=\alpha_td_{t,i}$
\ENDFOR
\STATE \textbf{Actions of Server:}
\STATE Aggregate $q_t=\frac{1}{M}\sum_{i=1}^Mq_{t,i}$
\COMMENT{average before adding the noise}
\STATE Draw $Y_t\sim\Ncal\left(0,I\sigma_t^2\right)$
\STATE Update $\tq_t=q_t+Y_t$
\COMMENT{the server is the one to induce privacy}
\STATE Update $w_{t+1}=\Pi_\K(w_t-\eta\tq_t)$
\STATE Update $x_{t+1}=(1-\frac{\alpha_{t+1}}{\alpha_{1:t+1}})x_t
+\frac{\alpha_{t+1}}{\alpha_{1:t+1}}w_{t+1}$
\ENDFOR
\STATE \textbf{Output:} $x_T$
\end{algorithmic}
\end{algorithm}

\subsection{Privacy Guarantees}
Here we establish the privacy guarantees of \Cref{alg:trust}.
Concretely, the following theorem shows how does the privacy of our algorithm depends on the variances of injected noise $\{Y_{t}\}_t$.
\begin{theorem}[Privacy Guarantees for \Cref{alg:trust}] 
\label{Thm:PrivacyTrust}
Let $\K\subset\reals^d$ be a convex set of diameter $D$, and $\{\fii{\cdot;z^i}\}_{z^i\in\Z_i}$ be a family of convex $G$-Lipschitz and $L$-smooth functions.
Then invoking~\Cref{alg:trust} with  noise distributions $Y_{t}\sim P_t=\Ncal\left(0,I\sigma_t^2\right)$, and any learning rate $\eta>0$; ensures that for any machine $i\in[M]$, the  query sequence $\{x_{t}\}_{t\in[T]}$ is $\left(\alpha,\frac{\alpha\rho^2}{2}\right)$-RDP for any $\alpha>1$, where:
$
\rho=\frac{2S}{M}\sqrt{\sum_{\tau=1}^T\frac{1}{\sigma_\tau^2}}~.
$
\end{theorem}
The full proof is in \Cref{proof:PrivacyTrust}.
\begin{proof}[Proof Sketch]
Fix $i\in[M]$, and assume that $\S_i$ and $\S_i'$ are neighboring datasets, meaning that there exists only a single index $\tau^*\in[T]$ where they differ, i.e.~that $z_{\tau^*,i}\neq z_{\tau^*,i}'$.
For  other machines $j\neq i$ assume that $\S_j$ are known and fixed.

Now, recalling that we obtain $\tq_{t}$ by adding a Gaussian noise $Y_{t}$ to $q_{t}$, we may use~\Cref{lem:div_gauss} and obtain:
\al\label{eq:RDP_SingleIterateImplicitTrust}
\diver{\tq_{t}(\S_i)\|\tq_{t}(\S_i')}=\frac{\alpha\Delta_t^2}{2\sigma_t^2}
\eal
Using that fact that $q_{t}=\frac{1}{M}\sum_{i=1}^M\sum_{\tau=1}^ts_{\tau,i}$ together with the bound $\norm{s_{\tau,i}}\leq S$, which holds for any $\tau\in[T]$ due to~\Cref{lem:bound_s}, we show 
that $\Delta_t\leq(2S/M)\cdot\mathbb{I}\{t\geq\tau^*\}$, with $\tau^*$ being the time step where $z_{\tau^*,i}\neq z_{\tau^*,i}'$, and $S:=G+2LD$.
We thus get:
\al\label{eq:RDP_SingleStepTrust}
\diver{\tq_{t}(\S_i)\|\tq_{t}(\S_i')}\leq\frac{2\alpha S^2}{M^2\sigma_t^2}\cdot\mathbb{I}\{t\geq\tau^*\}
\eal
Combining the above with \Cref{lem:rdp_copm}:
\al\label{eq:Q_Seq_Privacy_trust}
\diver{\{\tq_{\tau}(\S_i)\}_{\tau=1}^t\|\{\tq_{\tau}(\S_i')\}_{\tau=1}^t}\leq\frac{2\alpha S^2}{M^2}\sum_{\tau=1}^t\frac{1}{\sigma_\tau^2}
\eal
Finally, we use the post processing lemma for RDPs \citep{ren_div}, to bound the privacy of $\{x_\tau\}_{\tau=1}^t$ with the privacy of $\{\tq_\tau\}_{\tau=1}^t$.
\end{proof}

\subsection{Convergence Guarantees}
Next we state our main theorem for this setting.
\begin{theorem}\label{thm:trust}
Assume the same assumptions as in~\Cref{thm:untrust}.
Then upon invoking \Cref{alg:trust} with $\alpha_t=t, \beta_t=1/\alpha_t, \eta=\min\left\{\frac{\rho DM}{2ST\sqrt{d}},\frac{1}{4LT}\right\}$, and $\sigma_t^2=4S^2T/\rho^2M^2$, and any starting point $x_1\in\K$ and datasets $\{\S_i\in\Z_i^T\}_{i\in[M]}$, then for all $\alpha\geq1$ \Cref{alg:trust} satisfies $\left(\alpha,\frac{\alpha\rho^2}{2}\right)$-RDP for the query and gardient estimate sequences that the server produces, i.e.~$\{x_t\}_{t\in[T]},\{\tq_t\}_{t\in[T]}$.
Furthermore, if $\S_i$ consists of i.i.d.~samples from a distribution $\D_i$, then \Cref{alg:trust} guarantees:
\als
\R(x_T^{trust}):=\ExpB{\f{x_T^{trust}}}-\min_{x\in\K}\f{x} \\
\leq4D\left(\frac{G^*+2LD}{T}+\frac{2S\sqrt{d}}{\rho TM}+\frac{\tsigma}{\sqrt{TM}}\right)
\eals
\end{theorem}
The full proof is in \Cref{proof:trust}.
\begin{proof}[Proof Sketch]
The privacy guarantees follow directly from \Cref{Thm:PrivacyTrust}, and our choice of $\sigma_t^2$.
Regrading  convergence: in the spirit of $\mu^2$-SGD analysis \citep{mu_square}, we show that:
\al\label{eq:FinalBoundAlmostExplicitTrusted}
\alpha_{1:T}\R(x_T)\leq\frac{D^2}{\eta}+\eta\sum_{\tau=1}^T\Exp{\normsq{Y_\tau}} \non
+2\alpha_T DG^*+2D\sum_{\tau=1}^T\sqrt{\Exp{\normsq{\eps_\tau}}}
\eal
Where we denoted $\eps_\tau: = q_t - \alpha_t\df{x_t}$. Then, by utilizing Lemma~\ref{lem:bound_eps}, we show that  $\Exp{\normsq{\eps_\tau}}\leq \frac{1}{M}\Exp{\normsq{\eps_{\tau,i}}}\leq \tsigma^2 \tau/M$. Plugging these into the bound above and using $\alpha_t=t$, and our choices of $\eta$ and $\sigma_t^2$ implies the bound.
\end{proof}

\subsection{Experiments}
We ran \alg on MNSIT using the same specification as the previous section, but for the trusted server case.
We show our results in \Cref{tab:trust}.

\begin{table}[t]
\small
\begin{center}
\caption{\alg on MNIST with trusted server}
\begin{tabular}{c@{\hspace{0.3cm}}c@{\hspace{0.1cm}}c@{\hspace{0.3cm}}c@{\hspace{0.1cm}}c@{\hspace{0.3cm}}c@{\hspace{0.1cm}}c}
\toprule
 & \multicolumn{2}{c}{M=1, T=60,000} & \multicolumn{2}{c}{M=10, T=6,000} & \multicolumn{2}{c}{M=100, T=600} \\
$\rho$ & Loss & Accuracy & Loss & Accuracy & Loss & Accuracy \\
\midrule
 4 & 2.256 & 69.9\% & 2.256 & 69.7\% & 2.258 & 69.5\% \\
\midrule
 8 & 2.253 & 70.2\% & 2.253 & 70.1\% & 2.257 & 69.6\% \\
\midrule
16 & 2.252 & 70.4\% & 2.252 & 70.3\% & 2.256 & 69.7\% \\
\bottomrule
\end{tabular}
    \label{tab:trust}
    \end{center}
\end{table}

The First thing we can see, is that the results are the same for both cases when $M=1$, as this case in not federated, and thus there is no difference between the trusted and untrusted server cases.
When $M=10$, the loss doesn't change in a noticeable way, and the accuracy decreases a little.
When $M=100$, the loss decreases a little, but less than the untrusted server case, and the accuracy also decreases.
Another thing we can notice, is that with higher value of $\rho$, the trusted and untrusted server cases give more similar results.
That is because with lower privacy, the privacy type (trusted/untrusted) matters less.

\section{Conclusion}
\label{sec:conc}

We built upon the original $\mu^2$-SGD algorithm to accommodate differential privacy setting.
We extended the algorithm to the federated learning setting, and shown it is optimal in both the untrusted server and trusted server setting.
The optimal convergence is satisfied with linear computational complexity.

\section*{Impact Statement}
This paper presents work whose goal is to advance the field of Machine Learning, and especially the aspect of Privacy.
There are many potential societal consequences of our work, none which we feel must be specifically highlighted here.

\section*{Acknowledgement}
This research was partially supported by Israel PBC-VATAT, by the Technion Artificial Intelligent Hub (Tech.AI) and by the Israel Science Foundation (grant No. 447/20).

\bibliography{bib}
\bibliographystyle{icml2024}

\newpage
\appendix
\onecolumn
\section{Additional Theorems and Lemmas}
\label{sec:proof}

Here we provide additional theorems and lemmas that are used for our proofs.

\begin{theorem}[\citep{anytime}]\label{thm:anytime}
Let $f:\K\to\mathbb{R}$ be a convex function.
Also let $\{\alpha_t>0\}$ and $\{w_t\in\K\}$.
Let $\{x_t\}$ be the $\{\alpha_\tau\}_{\tau=1}^t$ weighted average of $\{w_\tau\}_{\tau=1}^t$, meaning: $x_t=\frac{1}{\alpha_{1:t}}\sum_{\tau=1}^t\alpha_\tau w_\tau$.
Then the following holds for all $t\geq1, x\in\K$:
\als
\alpha_{1:t}(\f{x_t}-\f{x})\leq\sum_{\tau=1}^t\alpha_\tau\dotprod{\df{x_\tau}}{w_\tau-x}
\eals
\end{theorem}
Note that the above theorem holds generally for any sequences of iterates $\{w_t\}_t$ with weighted averages $\{x_t\}_t$, and as a private case it holds for the sequences generated by Anytime-SGD.
Concretely, the theorem implies that the excess loss of the weighted average $x_t$ can be related to the weighted regret $\sum_{\tau=1}^t\alpha_\tau\dotprod{\df{x_\tau}}{w_\tau-x}$.

\begin{proof}[Proof of \Cref{thm:anytime}]
Proof by induction.
\\\textbf{Induction basis:} $t=1$
\als
\alpha_1(\f{x_1}-\f{x})\leq\alpha_1\dotprod{\df{x_1}}{x_1-x}=\alpha_1\dotprod{\df{x_1}}{w_1-x}
\eals
The inequality is from convexity of $f$, and the equality is because $x_1=w_1$.
\\\textbf{Induction assumption:} for some $t\geq1$:
\als
\alpha_{1:t}(\f{x_t}-\f{x})\leq\sum_{\tau=1}^t\alpha_\tau\dotprod{\df{x_\tau}}{w_\tau-x}
\eals
\\\textbf{Induction step:} proof for $t+1$:
\als
\alpha_{1:t+1}(\f{x_{t+1}}-\f{x})=&\alpha_{1:t}(\f{x_{t+1}}-\f{x_t}+\f{x_t}-\f{x})+\alpha_{t+1}(\f{x_{t+1}}-\f{x}) \\
\leq&\sum_{\tau=1}^t\alpha_\tau\dotprod{\df{x_\tau}}{w_\tau-x}+\alpha_{1:t}(\f{x_{t+1}}-\f{x_t})+\alpha_{t+1}(\f{x_{t+1}}-\f{x}) \\
\leq&\sum_{\tau=1}^t\alpha_\tau\dotprod{\df{x_\tau}}{w_\tau-x}+\dotprod{\df{x_{t+1}}}{\alpha_{1:t}(x_{t+1}-x_t)+\alpha_{t+1}(x_{t+1}-x)} \\
=&\sum_{\tau=1}^t\alpha_\tau\dotprod{\df{x_\tau}}{w_\tau-x}+\alpha_{t+1}\dotprod{\df{x_{t+1}}}{w_{t+1}-x}=\sum_{\tau=1}^{t+1}\alpha_\tau\dotprod{\df{x_\tau}}{w_\tau-x}
\eals
In the first equality we rearranged the terms, and added and subtracted the same thing, we then used the induction assumption, then used convexity of $f$ on both pairs and added them together, and finally we used the update rule of $x_{t+1}$: $\alpha_{1:t+1}x_{t+1}=\alpha_{1:t}x_t+\alpha_{t+1}w_{t+1}$, and added the last member of the sum to get our desired result.
\end{proof}

\begin{lemma}\label{lem:martingale}
Let $\{Z_t\}$ be a Martingale difference sequence w.r.t~a Filtration $\{\F_t\}_t$, i.e.~$\ExpB{Z_t\vert\F_{t-1}}=0$, then:
\als
\Exp{\normsq{\sum_{\tau=1}^tZ_\tau}}=\sum_{\tau=1}^t\Exp{\normsq{Z_\tau}}
\eals
\end{lemma}

\begin{proof}[Proof of \Cref{lem:martingale}]
\label{proof:martingale}
Proof by induction.
\\\textbf{Induction basis:} $t=1$
\als
\Exp{\normsq{\sum_{\tau=1}^1Z_\tau}}=\Exp{\normsq{Z_1}}=\sum_{\tau=1}^1\Exp{\normsq{Z_\tau}}
\eals
\\\textbf{Induction assumption:} for some $t\geq1$:
\als
\Exp{\normsq{\sum_{\tau=1}^tZ_\tau}}=\sum_{\tau=1}^t\Exp{\normsq{Z_\tau}}
\eals
\\\textbf{Induction step:} proof for $t+1$:
\als
\Exp{\normsq{\sum_{\tau=1}^{t+1}Z_\tau}}=\Exp{\normsq{\sum_{\tau=1}^{t+1}Z_\tau}}+2\ExpB{\dotprod{\sum_{\tau=1}^tZ_\tau}{Z_{t+1}}}+\Exp{\normsq{Z_{t+1}}}=\sum_{\tau=1}^t\Exp{\normsq{Z_\tau}}+\Exp{\normsq{Z_{t+1}}}=\sum_{\tau=1}^{t+1}\Exp{\normsq{Z_\tau}}
\eals
The first equality is square rules, then we used the induction assumption and the fact that $\ExpB{Z_{t+1}|Z_1,\ldots,Z_t}=0$, and finally we added the last term into the sum.
\end{proof}

\begin{lemma}[\citep{rdp}]\label{lem:rdp_copm}
If $\A_1,\ldots,\A_k$ are randomized algorithms satisfying $(\alpha,\epsilon_1)$-RDP, \ldots, $(\alpha,\epsilon_k)$-RDP, respectively, then their composition $(\A_1(\S)\ldots,\A_k(\S))$ is $(\alpha,\epsilon_1+\ldots,+\epsilon_k)$-RDP.
Moreover, the i'th algorithm $\A_i$, can be chosen on the basis of the outputs of the previous algorithms $\A_1,\ldots,\A_{i-1}$.
\end{lemma}

\begin{proof}[Proof of \Cref{lem:rdp_copm}]
Proof by induction.
\\\textbf{Induction basis:} $k=1$
\als
\diver{\A_1(\S)\|\A_1(\S')}\leq\epsilon_1
\eals
Because $\A_1$ is $(\alpha,\epsilon_1)$-RDP.
\\\textbf{Induction assumption:} for some $k\geq1$:
\als
\diver{\{\A_i(\S)\}_{i=1}^k\|\{\A_i(\S')\}_{i=1}^k}\leq\sum_{i=1}^k\epsilon_i
\eals
\\\textbf{Induction step:} proof for $k+1$
\als
\diver{\{\A_i(\S)\}_{i=1}^{k+1}\|\{\A_i(\S')\}_{i=1}^{k+1}}=&\frac{1}{\alpha-1}\log\left(\condExp{\left(\frac{\prob{\{\A_i(\S)\}_{i=1}^{k+1}}}{\prob{\{\A_i(\S')\}_{i=1}^{k+1}}}\right)^{\alpha-1}}{\A_i\sim\A_i(\S)}\right) \\
=&\frac{1}{\alpha-1}\log\left(\condExp{\left(\frac{\prob{\A_{k+1}(\S)|\{\A_i\}_{i=1}^k}\prob{\{\A_i(\S)\}_{i=1}^k}}{\prob{\A_{k+1}(\S')|\{\A_i\}_{i=1}^k}\prob{\{\A_i(\S')\}_{i=1}^k}}\right)^{\alpha-1}}{\A_i\sim\A_i(\S)}\right) \\
=&\frac{1}{\alpha-1}\log\left(\condExp{\left(\frac{\prob{\A_{k+1}(\S)|\{\A_i\}_{i=1}^k}}{\prob{\A_{k+1}(\S')|\{\A_i\}_{i=1}^k}}\right)^{\alpha-1}}{\A_i\sim\A_i(\S)}\right) \\
+&\frac{1}{\alpha-1}\log\left(\condExp{\left(\frac{\prob{\{\A_i(\S)\}_{i=1}^k}}{\prob{\{\A_i(\S')\}_{i=1}^k}}\right)^{\alpha-1}}{\A_i\sim\A_i(\S)}\right) \\
=&\diver{\A_{k+1}(\S)\|\A_{k+1}(\S')|\{\A_i\}_{i=1}^k}+\diver{\{\A_i(\S)\}_{i=1}^k\|\{\A_i(\S')\}_{i=1}^k} \\
\leq&\epsilon_{k+1}+\sum_{i=1}^k\epsilon_i=\sum_{i=1}^{k+1}\epsilon_i
\eals
The writing $\A_i\sim\A_i(\S)$ means that the output of $\A_i$ is distributed in the case that the dataset is $\S$.
The first equation is the definition of the R{\'{e}}nyi divergence, next we use conditional probability rules, then we use the fact that the output of $\A_{k+1}$ conditioned on the outputs of the previous algorithms is independent of the outputs of these previous algorithms, and let the product out of the log as a summery.
We then notice that each of the two members are divergences themselves, and finally, we invoke both the the fact that $\A_{k+1}$ is $(\alpha,\epsilon_{k+1})$-RDP, and the induction assumption, and add them together to the same sum.
\end{proof}

\begin{lemma}[Post Processing Lemma~\citep{ren_div}]\label{lem:post_proc}
Let $X,Y$ be random variables and $\A$ be a randomized or deterministic algorithm.
Then for all $\alpha\geq1$:
\als
\diver{\A(X)\|\A(Y)}\leq\diver{X\|Y}
\eals
\end{lemma}
We use this lemma to bound the privacy of $\{x_\tau\}_{\tau=1}^t$ with the privacy of $\{\tq_\tau\}_{\tau=1}^t$.

\begin{lemma}\label{lem:OGD+_Guarantees}Let $\eta>0$, and $\K\subset\reals^d$ be a convex domain of bounded diameter $D$, also let $\{\tq_t\in\reals^d\}_{t=1}^T$ be a sequence of arbitrary vectors.
Then for any starting point $w_1\in\reals^d$, and an update rule $w_{t+1}=\Pi_\K\left(w_t-\eta\tq_t\right)~,\forall t\geq 1$, the following holds $\forall x\in\K$:
\als
\sum_{\tau=1}^t\dotprod{\tq_\tau}{w_{\tau+1}-x}\leq\frac{D^2}{2\eta}-\frac{1}{2\eta}\sum_{\tau=1}^t\normsq{w_\tau-w_{\tau+1}}
\eals
\end{lemma}

\begin{proof}[Proof of \Cref{lem:OGD+_Guarantees}]
\label{proof:OGD+_Guarantees}
The update rule $w_{t+1}=\Pi_\K\left(w_t-\eta\tq_t\right)$ can be re-written as a convex optimization problem over $\K$:
\als
w_{t+1}=\Pi_\K\left(w_t-\eta\tq_t\right)=\underset{x\in\K}{\arg\min}\left\{\normsq{w_t-\eta\tq_t-x}\right\}=\underset{x\in\K}{\arg\min}\left\{\dotprod{\tq_t}{x-w_t}+\frac{1}{2\eta}\normsq{x-w_t}\right\}
\eals
The first equality is our update definition, the second is by the definition of the projection operator, and then we rewrite it in a way that does not affect the minimum point.

Now, since $w_{t+1}$ is the minimal point of the above convex problem, then from optimality conditions we obtain:
\als
\dotprod{\tq_t+\frac{1}{\eta}(w_{t+1}-w_t)}{x-w_{t+1}}\geq0,\quad\forall x\in\K
\eals
Re-arranging the above, we get that:
\als
\dotprod{\tq_t}{w_{t+1}-x}\leq\frac{1}{\eta}\dotprod{w_t-w_{t+1}}{w_{t+1}-x}=\frac{1}{2\eta}\normsq{w_t-x}-\frac{1}{2\eta}\normsq{w_{t+1}-x}-\frac{1}{2\eta}\normsq{w_t-w_{t+1}}
\eals
Where the equality is an algebraic manipulation.
After summing over $t$ we get: 
\als
\sum_{\tau=1}^t\dotprod{\tq_\tau}{w_{\tau+1}-x}
&\leq\frac{1}{2\eta}\sum_{\tau=1}^t\left(\normsq{w_\tau-x}-\normsq{w_{\tau+1}-x}-\normsq{w_\tau-w_{\tau+1}}\right) \\
&=\frac{\normsq{w_1-x}-\normsq{w_{t+1}-x}}{2\eta}-\frac{1}{2\eta}\sum_{\tau=1}^t\normsq{w_\tau-w_{\tau+1}}\leq\frac{D^2}{2\eta}-\frac{1}{2\eta}\sum_{\tau=1}^t\normsq{w_\tau-w_{\tau+1}}
\eals
Where the second line is due to splitting the sum into two sums, and using the fact that the first one is a telescopic sum, and lastly, we use the diameter of $\K$.
This establishes the lemma.
\end{proof}

\begin{lemma}\label{lem:smooth}
If $f:\K\to\reals$ is convex and $L$-smooth, and $x^*=\underset{x\in\K}{\arg\min}\{\f{x}\}$, then $\forall x\in\reals^d$:
\als
\normsq{\df{x}-\df{x^*}}\leq2L(\f{x}-\f{x^*})
\eals
\end{lemma}

\begin{proof}[Proof of \Cref{lem:smooth}]
Let us define a new function:
\als
h(x)=\f{x}-\f{x^*}-\dotprod{\df{x^*}}{x-x^*}
\eals
Since $f$ is convex and $L$-smooth, we know that:
\als
0\leq h(x)\leq\frac{L}{2}\normsq{x-x^*}
\eals
The gradient of this function is:
\als
\nabla h(x)=\df{x}-\df{x^*}
\eals
We can see that $h(x^*)=0, \nabla h(x^*)=0$, and that $x^*$ is the global minimum.
The function $h$ is also convex and $L$-smooth, since the gradient is the same as $f$ up to a constant translation.
We will add to the domain of $h$ to include all $\reals^d$, while still being convex and $L$-smooth.
Since $h$ is convex then:
\als
h(y)\geq h(x^*)+\dotprod{\nabla h(x^*)}{y-x^*}=0,\quad\forall y\in\reals^d
\eals
It is true even for points outside of the original domain, meaning that $x^*$ remains the global minimum even after this.
For a smooth function, $\forall x,y\in\reals^d$:
\als
h(y)\leq h(x)+\dotprod{\nabla h(x)}{y-x}+\frac{L}{2}\normsq{y-x}
\eals
By picking $y=x-\frac{1}{L}\nabla h(x)$, we get:
\als
h(x)-h(y)\geq\frac{1}{2L}\normsq{\nabla h(x)}
\eals
Rearranging, we get:
\als
\normsq{\nabla h(x)}\leq2L(h(x)-h(y))\leq2L\cdot h(x)
\eals
By using $x\in\K$ we get:
\als
\normsq{\df{x}-\df{x^*}}\leq2L(\f{x}-\f{x^*}-\dotprod{\df{x^*}}{x-x^*}
\eals
Since $x^*$ is the minimum point of $f$ then:
\als
\dotprod{\df{x^*}}{x-x^*}\geq0,\quad\forall x\in\K
\eals
Thus we get that:
\als
\normsq{\df{x}-\df{x^*}}\leq2L(\f{x}-\f{x^*})
\eals
\end{proof}

\begin{lemma}\label{lem:sum}
If $\A_t\leq\frac{1}{2T}\sum_{\tau=1}^T\A_\tau+\Bcal, \forall t\in[T]$, then $\A_t\leq2\Bcal, \forall t\in[T]$.
\end{lemma}

\begin{proof}[Proof of \Cref{lem:sum}]
\label{proof:sum}
Let's sum the inequalities:
\als
\sum_{t=1}^T\A_t\leq\sum_{t=1}^T\left(\frac{1}{2T}\sum_{\tau=1}^T\A_\tau+\Bcal\right)=\frac{1}{2}\sum_{\tau=1}^T\A_\tau+T\Bcal
\eals
The inequality is from the assumption, and the equality is because we sum constant values.
If we rearrange this we get:
\als
\sum_{\tau=1}^T\A_\tau\leq2T\Bcal
\eals
And then:
\als
\A_t\leq\frac{1}{2T}\sum_{\tau=1}^T\A_\tau+\Bcal\leq\Bcal+\Bcal=2\Bcal
\eals
\end{proof}

\section{Proofs of \Cref{sec:prelim}}

\subsection{Proof of \Cref{def:BoundedVar,def:BoundedSmoothnessVar}}
\label{proof:asmp}
Both claims use the same principle:
\als
\Exp{\normsq{X-\ExpB{X}}}\leq\Exp{\normsq{X}}
\eals
And so:
\als
\Exp{\normsq{\df{x;z}-\df{x}}}\leq&\Exp{\normsq{\df{x;z}}}\leq G^2 \\
\Exp{\normsq{(\df{x;z}-\df{x})-(\df{y;z}-\df{y})}}\leq&\Exp{\normsq{\df{x;z}-\df{y;z}}}\leq L^2\normsq{x-y}
\eals
Thus $\sigma\leq G$ and $\sigmal\leq L$.

\subsection{Proof of \Cref{lem:div_gauss}}
\label{proof:div_gauss}
Let us calculate $\diver{P\|Q}$ by definition.
\als
\diver{P\|Q}=&\frac{1}{\alpha-1}\log\left(\condExp{\left(\frac{P(X)}{Q(X)}\right)^{\alpha-1}}{X\sim P}\right)=\frac{1}{\alpha-1}\log\left(\condExp{e^{(\alpha-1)\frac{1}{2\sigma^2}\left(\normsq{X-\mu-\Delta}-\normsq{X-\mu}\right)}}{X\sim P}\right) \\
=&\frac{1}{\alpha-1}\log\left(\condExp{e^{\frac{\alpha-1}{2\sigma^2}\left(\normsq{\Delta}-2\dotprod{\Delta}{X-\mu}\right)}}{X\sim P}\right)=\frac{1}{\alpha-1}\log\left(e^{\frac{\alpha-1}{2\sigma^2}\normsq{\Delta}+\frac{1}{2}\sigma^2\frac{(\alpha-1)^2}{4\sigma^4}4\normsq{\Delta}}\right) \\
=&\frac{\normsq{\Delta}}{2\sigma^2}+\frac{(\alpha-1)\normsq{\Delta}}{2\sigma^2}=\frac{\alpha\normsq{\Delta}}{2\sigma^2}
\eals
The first equality is the definition of the R{\'{e}}nyi divergence, in the second we input the values of $P$ and $Q$, then we open the norm, and then we calculate the expectation using the moment generating function of a Gaussian random vector $\ExpB{e^{\dotprod{a}{X}}}=e^{\dotprod{\mu}{a}+\frac{1}{2}\sigma^2\normsq{a}}$, and finally the log and the exponent cancel each other, and we fix things up.

\subsection{Proof of \Cref{lem:dp_rdp}}
\label{proof:dp_rdp}
We will start with the first part.
To do it, we will show another property of the R{\'{e}}nyi divergence.
Holder's inequality states that for any $p,q\geq1$ such that $1/p+1/q=1$, we get $\|fq\|_1\leq\|f\|_p\|g\|_q$.
We will pick $p=\alpha, q=\frac{\alpha}{\alpha-1}, f=\frac{P}{Q^{1/q}}, g=Q^{1/q}$, and the norm to be on an arbitrary event $A$, and then:
\als
P(A)=\int_AP(x)dx\leq\left(\int_A\frac{(P(x))^\alpha}{(Q(x))^{\alpha-1}}dx\right)^{\frac{1}{\alpha}}\left(\int_AQ(x)dx\right)^{\frac{\alpha-1}{\alpha}}\leq\left(e^{\diver{P\|Q}}Q(A)\right)^{\frac{\alpha-1}{\alpha}}\leq\left(e^\epsilon Q(A)\right)^{\frac{\alpha-1}{\alpha}}
\eals
The equality is the definition of the probability of the event, the inequality is from Holder's, next we expand the integral over everything to get the R{\'{e}}nyi divergence, and finally we bound the R{\'{e}}nyi divergence by $\epsilon$.
We pick the event $A$ to be $A=\mO$, and then:
\als
\prob{\A(\S)=\mO}\leq\left(e^\epsilon\prob{\A(\S')=\mO}\right)^{1-\frac{1}{\alpha}}
\eals
If $\left(e^\epsilon\prob{\A(\S')=\mO}\right)^{1-\frac{1}{\alpha}}\leq\delta$ then $\prob{\A(\S)=\mO}\leq\delta$.
If $\left(e^\epsilon\prob{\A(\S')=\mO}\right)^{1-\frac{1}{\alpha}}>\delta$ then $\left(e^\epsilon\prob{\A(\S')=\mO}\right)^{-\frac{1}{\alpha}}<\delta^{-\frac{1}{\alpha-1}}=e^{\frac{\log(1/\delta)}{\alpha-1}}$, and then:
\als
\prob{\A(\S)=\mO}\leq e^{\epsilon+\frac{\log(1/\delta)}{\alpha-1}}\prob{\A(\S')=\mO}
\eals
If we combine both cases, we showed that:
\als
\prob{\A(\S)=\mO}\leq\max\left\{e^{\epsilon+\frac{\log(1/\delta)}{\alpha-1}}\prob{\A(\S')=\mO},\delta\right\}\leq e^{\epsilon+\frac{\log(1/\delta)}{\alpha-1}}\prob{\A(\S')=\mO}+\delta
\eals
Now for the second part.
Since $\A$ is $\left(\alpha,\frac{\alpha\rho^2}{2}\right)$-RDP for every $\alpha\geq1$, then it is also $\left(\frac{\alpha\rho^2}{2}+\frac{\log(1/\delta)}{\alpha-1},\delta\right)$-DP, for every $\alpha\geq1, \delta\in(0,1)$.
We will pick the value that minimize this expression $\alpha=1+\frac{1}{\rho}\sqrt{2\log(1/\delta)}$, and get that for every $\delta\in(0,1)$, $\A$ is $\left(\frac{\rho^2}{2}+\rho\sqrt{2\log(1/\delta)},\delta\right)$-DP.

\section{Proofs of \Cref{sec:mult}}

\subsection{Proof of \Cref{lem:Q_Srelation}}
\label{proof:Q_Srelation}
First note that for $t=1$ we have $\alpha_t=1$, and therefore we have, $s_{1,i}:=g_{1,i}=d_{1,i}=q_{1,i}$, which implies that the lemma holds for $t=1$.
For that case of $t>1$, our choices for $\{\alpha_t,\beta_t\}$ imply that $\beta_{t+1}\alpha_{t+1}=1, (1-\beta_{t+1})\alpha_{t+1}=\alpha_t$, which prove~\Cref{eq:q_tUpdate1}.
Unrolling the above equation yields $q_{t+1,i}=\sum_{\tau=1}^{t+1}s_{\tau,i}$, which establishes the first part of the lemma.
For the second part regarding $\eps_{t,i}$, note that our choice $\alpha_t=t$, together with the definition of $\bs_t$ (\Cref{Eq:S_tBarS_t}) implies that we can write $\bs_{t,i}=\alpha_t\bg_{t,i}-\alpha_{t-1}\bg_{t-1,i}$ (where for consistency we denote $\alpha_0:=0$).
Therefore, for any $t$:
\als
\sum_{\tau=1}^t\bs_{\tau,i}=\alpha_t\bg_{t,i}
\eals
Using the above together with $q_t=\sum_{\tau}^ts_{\tau,i}$, immediately gives:
\als
\eps_{t,i}:=q_{t,i}-\alpha_t\bg_{t,i}=\sum_{\tau=1}^t(s_{\tau,i}-\bs_{\tau,i})
\eals
Which establishes the proof.

\subsection{Proof of \Cref{lem:bound_s}}
\label{proof:bound_s}
Before we begin, we shall bound the difference between consecutive query points:
\als
x_t-x_{t-1}=x_t-\frac{\alpha_{1:t}x_t-\alpha_tw_t}{\alpha_{1:t-1}}=\frac{\alpha_t}{\alpha_{1:t-1}}(w_t-x_t)
\eals
Where the first equality is due to $\alpha_{1:t}x_t=\alpha_{1:t-1}x_{t-1}+\alpha_tw_t$.
Using the above enables to bound the following scaled difference:
\al\label{eq:BoundConsec}
\alpha_{t-1}\norm{x_t-x_{t-1}}=\left(\frac{\alpha_{t-1}\alpha_t}{\alpha_{1:t-1}}\right)\norm{w_t-x_t}\leq2D
\eal
Where we use that fact that $\alpha_t=t$, which implies $\alpha_{t-1}\alpha_t=2\alpha_{1:t-1}$, as well as the bounded diameter assumption.

First part:
\als 
\norm{s_{t,i}}=&\norm{g_{t,i}+\alpha_{t-1}(g_{t,i}-\tg_{t-1,i})}\leq\norm{g_{t,i}}+\alpha_{t-1}\norm{g_{t,i}-\tg_{t-1,i}} \\
=&\norm{\df{x_t;z_{t,i}}}+\alpha_{t-1}\norm{\df{x_t;z_{t,i}}-\df{x_{t-1};z_{t,i}}} \\
\leq&G+\alpha_{t-1}L\norm{x_t-x_{t-1}}\leq G+2LD:=S
\eals
The first equality is from the definition of $s_{t,i}$, then we use the triangle inequality, then explicitly employ the definitions of $g_{t,i},\tg_{t-1,i}$, next we use Lipschitz and smoothness, and finally, we employ~\Cref{eq:BoundConsec}, and the definition of $S$.

Second part:
\als
\Exp{\normsq{s_{t,i}-\Bar{s}_{t,i}}}=&\Exp{\normsq{(g_{t,i}-\Bar{g}_{t,i})+\alpha_{t-1}((g_{t,i}-\Bar{g}_{t,i})-(\tg_{t-1,i}-\Bar{g}_{t-1,i}))}} \\
\leq&\left(\sqrt{\Exp{\normsq{g_{t,i}-\Bar{g}_{t,i}}}}+\alpha_{t-1}\sqrt{\Exp{\normsq{(g_{t,i}-\Bar{g}_{t,i})-(\tg_{t-1,i}-\Bar{g}_{t-1,i})}}}\right)^2 \\
\leq&\left(\sigma+\alpha_{t-1}\sigmal\norm{x_t-x_{t-1}}\right)^2\leq(\sigma+2\sigma_LD)^2:=\tsigma^2 
\eals
The first equality is by the definition of $s_{t,i}-\Bar{s}_{t,i}$, then we use the inequality:
 $\Exp{\normsq{X+Y}}\leq\left(\sqrt{\Exp{\normsq{X}}}+\sqrt{\Exp{\normsq{Y}}}\right)^2$, 
which holds since:
\als
\Exp{\normsq{X+Y}}&=\Exp{\normsq{X}+2\dotprod{X}{Y}+\normsq{Y}} \\
&\leq\Exp{\normsq{X}}+\sqrt{\Exp{\normsq{X}\cdot\normsq{Y}}}+\Exp{\normsq{Y}}=\left(\sqrt{\Exp{\normsq{X}}}+\sqrt{\Exp{\normsq{Y}}}\right)^2
\eals
Then, we use the bounded variance assumption as well as the bounded smoothness variance assumption (\Cref{def:BoundedVar,def:BoundedSmoothnessVar}), and finally, we employ~\Cref{eq:BoundConsec}, and the definition of $\tsigma$. 

\subsection{Proof of \Cref{lem:bound_eps}}
\label{proof:bound_eps}
Using the previous lemmas will naturally allow us to bound $\eps_{t,i}$:
\als
\Exp{\normsq{\eps_{t,i}}}=\Exp{\normsq{\sum_{\tau=1}^t(s_{\tau,i}-\Bar{s}_{\tau,i})}}=\sum_{\tau=1}^t\Exp{\normsq{s_{\tau,i}-\Bar{s}_{\tau,i}}}\leq\sum_{\tau=1}^t\tsigma^2=\tsigma^2t
\eals
The first equality is~\Cref{lem:Q_Srelation}, then we use~\Cref{lem:martingale}, and finally~\Cref{lem:bound_s}.

\section{Proof of Privacy}
\label{sec:private_proof}

Here we will provide the proof of the privacy of \Cref{alg:untrust,alg:trust}.

\subsection{Proof of \Cref{Thm:Privacy}}
\label{proof:Privacy}
Let us look at a single machine $i$.
Let $\S_i,\S_i'$ be \emph{neighbouring datasets} of $T$ samples which differ on a single datapoint; i.e.~assume there exists $\tau^*\in[T]$ such that $\S_i:=\{z_{1,i},z_{2,i},\ldots,z_{\tau^*,i},\ldots,z_{T,i}\}$ and $\S_i':=\{z_{1,i},z_{2,i},\ldots,z_{\tau^*,i}',\ldots,z_{T,i}\}$, and $z_{\tau^*}\neq z_{\tau^*}'$.

\textbf{Notation:} We will employ $q_{t,i}(\S_i),\tq_{t,i}(\S_i),x_t(\S),s_{t,i}(\S_i)$ to denote the resulting values of these quantities when we invoke~\Cref{alg:untrust} with the datasets $\S$, we will similarly denote $q_{t,i}(\S_i'),\tq_{t,i}(\S_i'),x_t(\S'),s_{t,i}(\S_i')$.

\textbf{First Part:}
Here we will bound the R{\'{e}}nyi Divergence between $\tq_{t,i}(\S_i)$ and $\tq_{t,i}(\S_i')$ conditioned on the values of $\{x_1,\{\tq_\tau\}_{\tau=1}^{t-1}\}$.
Note that given $\{x_1,\{\tq_\tau\}_{\tau=1}^{t-1}\}$ we can directly compute the values of $\{w_\tau\}_{\tau=1}^t$ and therefore also $\{x_\tau\}_{\tau=1}^t$ (see~\Cref{alg:untrust}).
Thus, we may assume that prior to the computation of $q_{t,i}$, we are given $\{x_\tau\}_{\tau=1}^t$.

Note that given the query history $x_1,\ldots,x_t$, and the dataset $\S_i$, $q_{t,i}(\S_i)$ is known (respectively $q_{t,i}(\S_i')$ is known given the dataset $\S_i'$ and the query history).
Now, since we obtain $\tq_{t,i}$ by adding a Gaussian noise $Y_{t,i}$ to $q_{t,i}$, we may use~\Cref{lem:div_gauss} and obtain:
\als
\diver{\tq_{t,i}(\S_i)\|\tq_{t,i}(\S_i')}=\frac{\alpha\Delta_{t,i}^2}{2\sigma_{t,i}^2}
\eals
Where $\sigma_{t,i}^2$ is the variance of the noise $Y_{t,i}$, and $\Delta_{t,i}:=\norm{q_{t,i}(\S_i)-q_{t,i}(\S_i')}$.
\\\textbf{Next, we will bound $\Delta_{t,i}$:} First note that given $x_1,\ldots,x_t$ then by the definition of $s_{t,i}$ (\Cref{Eq:S_tBarS_t}) it only depends on $z_{t,i}$, i.e.:
\als
s_{t,i}=&g_{t,i}+\alpha_{t-1,i}(g_{t,i}-\tg_{t-1,i})=\df{x_t;z_{t,i}}+\alpha_{t-1}(\df{x_t;z_{t,i}}-\df{x_{t-1};z_{t,i}})
\eals
Thus, for any $t\in[T]$, then given $x_1,\ldots x_t$ we have:
\als
s_{t,i}(\S_i)-s_{t,i}(\S_i')=0,\quad\forall t\neq\tau^*
\eals
where we used the fact that the datasets $\S_i,\S_i'$ only differ on the ${\tau^*}$'th sample.
Using the above together with the fact that $q_{t,i}=\sum_{\tau=1}^t s_{\tau,i}$ (see~\Cref{lem:Q_Srelation}), implies the following $\forall t$:
\als 
\Delta_{t,i}:=&\norm{q_{t,i}(\S_i)-q_{t,i}(\S_i')}=\mathbb{I}\{t\geq\tau\}\norm{s_{\tau^*,i}(\S_i)-s_{\tau^*,i}(\S_i')} \\
\leq&\mathbb{I}\{t\geq\tau^*\}(\norm{s_{\tau^*,i}(\S_i)}+\norm{s_{\tau^*,i}(\S_i')})\leq2S\cdot\mathbb{I}\{t\geq\tau^*\}
\eals
Where $\mathbb{I}\{t\geq\tau^*\}$ denotes the indicator function (i.e.~it equals zero if $t<\tau^*$, and is equal to $1$ otherwise).
Moreover, we also used~\Cref{lem:bound_s} to bound the norms of $s_{\tau^*}(\S),s_{\tau^*}(\S')$.
Using the above together with~\Cref{eq:RDP_SingleIterateImplicit} we can bound for any $t\in[T]$:
\als
\diver{\tq_{t,i}(\S_i)\|\tq_{t,i}(\S_i')}=\frac{\alpha\Delta_{t,i}^2}{2\sigma_{t,i}^2}\leq\frac{2\alpha S^2}{\sigma_{t,i}^2}\cdot\mathbb{I}\{t\geq\tau^*\}
\eals

\textbf{Second Part:}
Here we will bound the R{\'{e}}nyi Divergence of the sequence $\{\tq_{t,i}\}_{t\in[T]}$.
In~\Cref{eq:RDP_SingleStep} we bound the R{\'{e}}nyi Divergence of the $t$'th step of our algorithm, which produces $\tq_{t,i}$ based on the previously computed $\{x_1,\{\tq_{\tau,i}\}_{\tau=1}^{t-1}\}$ as well as based on the dataset, and the injected noise $Y_{t,i}$.
We shall relate this $t$'th step update as a procedure, and denote by $\A_t$.
Thus, the procedure $\A_t$ computes its output $\tq_{t,i}$ based on the outputs of the previous procedures $\A_1,\ldots\A_{t-1}$, as well as on the dataset and the initialization point $x_1$ (which does not depend on the dataset).
This allows us to directly apply the RDP Composition rule in~\Cref{lem:rdp_copm}, together with~\Cref{eq:RDP_SingleStep}, yielding $\forall t\in[T]$:
\al \label{eq:Q_Seq_Privacy_untrustApp}
\diver{\{\tq_{\tau,i}(\S_i)\}_{\tau=1}^t\|\{\tq_{\tau,i}(\S_i')\}_{\tau=1}^t}\leq\sum_{\tau=1}^t\frac{\alpha\Delta_{\tau,i}^2}{2\sigma_{\tau,i}^2}\leq2\alpha S^2\sum_{\tau=\tau^*}^t\frac{1}{\sigma_{\tau,i}^2}\leq2\alpha S^2\sum_{\tau=1}^t\frac{1}{\sigma_{\tau,i}^2}
\eal
Thus, the theorem follows as per~\Cref{def:RDP} of RDP, and by~\Cref{eq:Q_Seq_Privacy_untrustApp}.

\subsection{Proof of \Cref{Thm:PrivacyTrust}}
\label{proof:PrivacyTrust}
The proof follows the same steps as the proof of \Cref{proof:Privacy}.
Since we obtain $\tq_t$ by adding a Gaussian noise $Y_t$ to $q_t$, we may use~\Cref{lem:div_gauss} and obtain:
\als
\diver{\tq_t(\S)\|\tq_t(\S')}=\frac{\alpha\Delta_t^2}{2\sigma_t^2}
\eals
Where $\sigma_t^2$ is the variance of the noise $Y_t$, and $\Delta_t:=\norm{q_t(\S)-q_t(\S')}$.
Since $q_{t}=\frac{1}{M}\sum_{i=1}^M\sum_{\tau=1}^ts_{\tau,i}$ and only $z_{\tau^*,i}$ is different between $\S,\S'$, we get:
\als 
\Delta_t:=&\norm{q_t(\S)-q_t(\S')}=\frac{1}{M}\mathbb{I}\{t\geq\tau\}\norm{s_{\tau^*}(\S)-s_{\tau^*}(\S')} \\
\leq&\frac{1}{M}\mathbb{I}\{t\geq\tau^*\}(\norm{s_{\tau^*}(\S)}+\norm{s_{\tau^*}(\S')})\leq\frac{2S}{M}\cdot\mathbb{I}\{t\geq\tau^*\}
\eals
Using the above together with \Cref{eq:RDP_SingleIterateImplicitTrust} we can bound for any $t\in[T]$:
\als
\diver{\tq_t(\S)\|\tq_t(\S')}=\frac{\alpha\Delta_t^2}{2\sigma_t^2}\leq\frac{2\alpha S^2}{M^2\sigma_t^2}\cdot\mathbb{I}\{t\geq\tau^*\}
\eals
As before, we apply the RDP Composition rule in~\Cref{lem:rdp_copm}, together with~\Cref{eq:RDP_SingleStepTrust}, yielding $\forall t\in[T]$:
\als
\diver{\{\tq_\tau(\S)\}_{\tau=1}^t\|\{\tq_\tau(\S')\}_{\tau=1}^t}\leq\sum_{\tau=1}^t\frac{\alpha\Delta_\tau^2}{2\sigma_\tau^2}\leq\frac{2\alpha S^2}{M^2}\sum_{\tau=\tau^*}^t\frac{1}{\sigma_\tau^2}\leq\frac{2\alpha S^2}{M^2}\sum_{\tau=1}^t\frac{1}{\sigma_\tau^2}
\eals
To bound the RDP of the $\{x_t\}$ sequence, recall from~\Cref{alg:trust} that given $\{x_1,\{\tq_\tau\}_{\tau=1}^{t-1}\}$ we can directly compute the values of $\{w_\tau\}_{\tau=1}^t$ and therefore also $\{x_\tau\}_{\tau=1}^t$ (see~\Cref{alg:trust}).
This allows us to bound privacy of the queries $\{x_\tau\}_{\tau=1}^t$ as follows for any $t\in[T]$:
\al \label{eq:Q_Seq_Privacy_trustAppendix}
\diver{\{x_\tau(\S)\}_{\tau=1}^t\|\{x_\tau(\S')\}_{\tau=1}^t}\leq\diver{\{\tq_\tau(\S)\}_{\tau=1}^t\|\{\tq_\tau(\S')\}_{\tau=1}^t}\leq\frac{2\alpha S^2}{M^2}\sum_{\tau=1}^t\frac{1}{\sigma_\tau^2}
\eal 
Where we have used the Post-Processing property of RDP appearing in \Cref{lem:post_proc}.

Thus, the theorem follows as per~\Cref{def:RDP} of RDP, and by~\Cref{eq:Q_Seq_Privacy_trustAppendix}.

\section{Proof of Convergence}
\label{sec:converge_proof}

Here we will provide the proof of the convergence of \Cref{alg:untrust,alg:trust}.

Our starting point is the following lemma regarding the update rule $w_{t+1}=\Pi_\K\left(w_t-\eta\tq_t\right)$ that we employ in~\Cref{alg:untrust,alg:trust}.
Employing \Cref{lem:OGD+_Guarantees} together with the Anytime guarantees, i.e.~\Cref{thm:anytime}, and recalling that $x_t$ is a weighted average of the $\{w_t\}_t$ sequence, implies that we may bound the excess loss $\R(x_t):=\ExpB{\f{x_t}}-\underset{x\in\K}{\min}\{\f{x}\}$ of~\Cref{alg:untrust,alg:trust} as follows:
\als
\alpha_{1:t}\R(x_t)=&\alpha_{1:t}(\ExpB{\f{x_t}}-\f{x^*})\leq\sum_{\tau=1}^t\ExpB{\alpha_\tau\dotprod{\df{x_\tau}}{w_\tau-x^*}}=\sum_{\tau=1}^t\ExpB{\dotprod{\tq_\tau-\eps_\tau-Y_\tau}{w_\tau-x^*}} \\
=&\sum_{\tau=1}^t\ExpB{\dotprod{\tq_\tau}{w_\tau-x^*}}-\sum_{\tau=1}^t\ExpB{\dotprod{\eps_\tau}{w_\tau-x^*}} \\
=&\sum_{\tau=1}^t\ExpB{\dotprod{\tq_\tau}{w_{\tau+1}-x^*}}+\sum_{\tau=1}^t\ExpB{\dotprod{\tq_\tau}{w_\tau-w_{\tau+1}}}-\sum_{\tau=1}^t\ExpB{\dotprod{\eps_\tau}{w_\tau-x^*}} \\
\leq&\frac{D^2}{2\eta}-\frac{1}{2\eta}\sum_{\tau=1}^t\Exp{\normsq{w_\tau-w_{\tau+1}}}+\sum_{\tau=1}^t\ExpB{\dotprod{\tq_\tau}{w_\tau-w_{\tau+1}}}-\sum_{\tau=1}^t\ExpB{\dotprod{\eps_\tau}{w_\tau-x^*}} \\
\leq&\frac{D^2}{2\eta}+\underset{\rA}{\underbrace{\sum_{\tau=1}^t\ExpB{\dotprod{\alpha_\tau\df{x_\tau}-\alpha_\tau\df{x^*}+Y_\tau}{w_\tau-w_{\tau+1}}-\frac{1}{2\eta}\normsq{w_\tau-w_{\tau+1}}}}} \\
+&\underset{\rB}{\underbrace{\sum_{\tau=1}^t\ExpB{\dotprod{\alpha_\tau\df{x^*}}{w_\tau-w_{\tau+1}}}}}+\underset{\rC}{\underbrace{\sum_{\tau=1}^t\ExpB{\dotprod{-\eps_\tau}{w_{\tau+1}-x^*}}}}
\eals
The first inequality is due to \Cref{thm:anytime}, the next one is due to $\tq_\tau=\alpha_\tau\df{x_\tau}+\eps_\tau+Y_\tau$, then we use the fact that $Y_\tau$ is zero mean and is independent of $w_\tau$, and after that we add and subtract $\pm\dotprod{\tq_\tau}{w_{\tau+1}}$, and then use~\Cref{lem:OGD+_Guarantees}.
Next, we reuse the definition of $\tq_\tau$ and move terms together plus adding and subtracting $\pm\alpha_\tau\df{x^*}$.
Next we separately bound the above terms $\rA,\rB,\rC$.
\\\textbf{Bounding $\rA$}
This term can be bounded as follows:
\als
\rA:=&\sum_{\tau=1}^t\ExpB{\dotprod{\alpha_\tau\df{x_\tau}-\alpha_\tau\df{x^*}+Y_\tau}{w_\tau-w_{\tau+1}}-\frac{1}{2\eta}\sum_{\tau=1}^t\normsq{w_\tau-w_{\tau+1}}} \\
\leq&\frac{\eta}{2}\sum_{\tau=1}^t\Exp{\normsq{\alpha_\tau(\df{x_\tau}-\df{x^*})+Y_\tau}}=\frac{\eta}{2}\sum_{\tau=1}^t\alpha_\tau^2\Exp{\normsq{\df{x_\tau}-\df{x^*}}}+\frac{\eta}{2}\sum_{\tau=1}^t\Exp{\normsq{Y_\tau}}
\eals
First, we use Young's inequality, and equality holds since $Y_\tau$ is independent of $x_\tau$, and has zero mean.
\\\textbf{Bounding $\rB$}
This term can be written as follows:
\als
\rB:=&\sum_{\tau=1}^t\ExpB{\dotprod{\alpha_\tau\df{x^*}}{w_\tau-w_{\tau+1}}}=\sum_{\tau=1}^t(\alpha_\tau-\alpha_{\tau-1})\ExpB{\dotprod{\df{x^*}}{w_\tau}}-\alpha_t\ExpB{\dotprod{\df{x^*}}{w_{t+1}}} \\
=&\sum_{\tau=1}^t(\alpha_\tau-\alpha_{\tau-1})\ExpB{\dotprod{\df{x^*}}{w_\tau-w_{t+1}}}\leq\sum_{\tau=1}^t(\alpha_\tau-\alpha_{\tau-1})\norm{\df{x^*}}\ExpB{\norm{w_\tau-w_{t+1}}} \\
\leq&D\norm{\df{x^*}}\sum_{\tau=1}^t(\alpha_\tau-\alpha_{\tau-1})=\alpha_t D\norm{\df{x^*}}=\alpha_t D G^*
\eals
The first equality is rearrangement of the sum while defining $\alpha_0=0$, then we put the last term into the sum, and use Cauchy-Schwartz.
Finally we use the diameter bound and telescope the  sum.
Note that we use the definition $G^*:=\norm{\df{x^*}}$, and that $G^*\in[0,G]$.
\\\textbf{Bounding $\rC$}
This term is bounded as follows:
\als
\rC:=\sum_{\tau=1}^t\ExpB{\dotprod{\eps_\tau}{w_{\tau+1}-x^*}}\leq\sum_{\tau=1}^t\ExpB{\norm{\eps_\tau}\cdot\norm{w_{\tau+1}-x^*}}\leq D\sum_{\tau=1}^t\sqrt{\Exp{\normsq{\eps_\tau}}}
\eals
The first inequality is Cauchy-Schwartz, and the second is using the bounded diameter and Jensen's inequality w.r.t.~concave function $\sqrt{\cdot}$.
In total, we get that:
\al\label{eq:StageMiddle}
\alpha_{1:t}\R(x_t)\leq\frac{D^2}{2\eta}+\frac{\eta}{2}\sum_{\tau=1}^t\alpha_\tau^2\Exp{\normsq{\df{x_\tau}-\df{x^*}}}+\frac{\eta}{2}\sum_{\tau=1}^t\Exp{\normsq{Y_\tau}}+\alpha_t DG^*+D\sum_{\tau=1}^t\sqrt{\Exp{\normsq{\eps_\tau}}}
\eal
Using \Cref{lem:smooth}, together with $\alpha_\tau^2=\tau^2\leq2\frac{\tau(\tau+1)}{2}=2\alpha_{1:\tau}$:
\als
\alpha_\tau^2\Exp{\normsq{\df{x_\tau}-\df{x^*}}}\leq4\alpha_{1:\tau}L\R(x_\tau)
\eals
Using the above and plugging it back into~\Cref{eq:StageMiddle} yields:
\als
\alpha_{1:t}\R(x_t)\leq\frac{D^2}{2\eta}+2\eta L\sum_{\tau=1}^t\alpha_{1:\tau}\R(x_\tau)+\frac{\eta}{2}\sum_{\tau=1}^t\Exp{\normsq{Y_\tau}}+\alpha_tDG^*+D\sum_{\tau=1}^t\sqrt{\Exp{\normsq{\eps_\tau}}}
\eals
We will now increase $t$ to $T$ in the right-hand-side (since the RHS is monotonically non-decreasing with $t$) and use our choice of the learning rate which implies $\eta\leq\frac{1}{4LT}$, to get that for any $t\leq T$,
\als
\alpha_{1:t}\R(x_t)\leq\frac{1}{2T}\sum_{\tau=1}^T\alpha_{1:\tau}\R(x_\tau)+\frac{D^2}{2\eta}+\frac{\eta}{2}\sum_{\tau=1}^T\Exp{\normsq{Y_\tau}}+\alpha_TDG^*+D\sum_{\tau=1}^T\sqrt{\Exp{\normsq{\eps_\tau}}}
\eals
Now, using the \Cref{lem:sum} with:
\als
\A_t=&\alpha_{1:t}\R(x_t) \\
\Bcal=&\frac{D^2}{2\eta}+\frac{\eta }{2}\sum_{\tau=1}^T\Exp{\normsq{Y_\tau}}+\alpha_T DG^*+D\sum_{\tau=1}^T\sqrt{\Exp{\normsq{\eps_\tau}}}
\eals
We finally get:
\als
\alpha_{1:T}\R(x_T)\leq\frac{D^2}{\eta}+\eta\sum_{\tau=1}^T\Exp{\normsq{Y_\tau}}+2\alpha_T DG^*+2D\sum_{\tau=1}^T\sqrt{\Exp{\normsq{\eps_\tau}}}
\eals
We can see that the bound depends on our learning rate $\eta$, the estimation error $\eps_\tau$ and the injected noise $Y_\tau$.

Using \Cref{lem:bound_eps}, and the fact the $\{\eps_{t,i}\}_{i\in[M]}$ are independent given the queries $\{x_\tau\}_{\tau\in[t]}$, since each of them depends on a different dataset, we get that: $\Exp{\normsq{\eps_t}}\leq\frac{\tsigma^2t}{M}$.
Now we can bound:
\als
\sum_{\tau=1}^T\sqrt{\Exp{\normsq{\eps_\tau}}}\leq\sum_{\tau=1}^T\sqrt{\frac{\tsigma^2\tau}{M}}\leq\frac{\tsigma T^{1.5}}{\sqrt{M}}
\eals
The bound on $Y_\tau$ will be different between our two versions.

\subsection{Proof of \Cref{thm:untrust}}
\label{proof:untrust}
Since we pick $\sigma_{t,i}^2=\frac{4S^2T}{\rho^2}$ and $Y_\tau=\frac{1}{M}\sum_{i=1}^MY_{\tau,i}$, then:
\als
\sum_{\tau=1}^T\Exp{\normsq{Y_\tau}}=\frac{1}{M^2}\sum_{\tau=1}^T\sum_{i=1}^M\Exp{\normsq{Y_{\tau,i}}}=\frac{1}{M^2}\sum_{\tau=1}^T\sum_{i=1}^Md\sigma_{t,i}^2=\frac{4S^2T^2d}{\rho^2M}
\eals

When we put our bounds together in~\Cref{eq:FinalBoundAlmostExplicitUntrusted} we get:
\als
\alpha_{1:T}\R(x_T)\leq\frac{D^2}{\eta}+\frac{4\eta S^2T^2d}{\rho^2M}+2\alpha_TDG^*+\frac{2D\tsigma T^{1.5}}{\sqrt{M}}
\eals
Finally, dividing by $\alpha_{1:T}=T(T+1)/2\geq T^2/2$ and recalling our choice $\eta=\min\left\{\frac{\rho D\sqrt{M}}{2ST\sqrt{d}},\frac{1}{4LT}\right\}$, using $\eta\leq\frac{\rho D\sqrt{M}}{2ST\sqrt{d}}, \frac{1}{\eta}\leq\frac{2ST\sqrt{d}}{\rho D\sqrt{M}}+4LT$, and that $\alpha_T=T$, yields the final bound:
\als
\R(x_T)\leq4D\left(\frac{G^*+2LD}{T}+\frac{2S\sqrt{d}}{\rho T\sqrt{M}}+\frac{\tsigma}{\sqrt{TM}}\right)
\eals
Thus concluding the proof.

\subsection{Proof of \Cref{thm:trust}}
\label{proof:trust}
Since we pick $\sigma_t^2=\frac{4S^2T}{\rho^2M^2}$, then:
\als
\sum_{\tau=1}^T\Exp{\normsq{Y_\tau}}=\sum_{\tau=1}^T\sum_{i=1}^Md\sigma_t^2=\frac{4S^2T^2d}{\rho^2M^2}
\eals

When we put our bounds together in~\Cref{eq:FinalBoundAlmostExplicitTrusted} we get:
\als
\alpha_{1:T}\R(x_T)\leq\frac{D^2}{\eta}+\frac{4\eta S^2T^2d}{\rho^2M^2}+2\alpha_TDG^*+\frac{2D\tsigma T^{1.5}}{\sqrt{M}}
\eals
Finally, dividing by $\alpha_{1:T}=T(T+1)/2\geq T^2/2$ and recalling our choice $\eta=\min\left\{\frac{\rho DM}{2ST\sqrt{d}},\frac{1}{4LT}\right\}$, using $\eta\leq\frac{\rho DM}{2ST\sqrt{d}}, \frac{1}{\eta}\leq\frac{2ST\sqrt{d}}{\rho DM}+4LT$, and that $\alpha_T=T$, yields the final bound:
\als
\R(x_T)\leq4D\left(\frac{G^*+2LD}{T}+\frac{2S\sqrt{d}}{\rho TM}+\frac{\tsigma}{\sqrt{TM}}\right)
\eals
Thus concluding the proof.
\section{Review of \citep{private_fed}}
\label{sec:other}

In \citep{private_fed} they propose an algorithm for private federated learning.
In this paper they assume to have $M$ machines, and  $n$ is the size of the dataset used by each machine, and therefore the total dataset is of size $|\S| = nM$.
Moreover, they denote by $R$ to total number of update rounds, and by $K$ the minibatch size used by each machine in every round, totalling $RKM$ gradient computations overall.
In theorem D.2 of this paper we can see that for the convex case of ISRL-DP (untrusted server) they use $R=\frac{\beta D\sqrt{M}}{L}\min\left\{\sqrt{n},\frac{\epsilon_0n}{\sqrt{d}}\right\}+\min\left\{nM,\frac{\epsilon_0^2n^2M}{d}\right\}$ time steps, and batch size of $K\geq\frac{\epsilon_0n}{4\sqrt{2R\ln(2/\delta_0)}}$.
The total number of gradient computations each silo use is $RK$, and the total number of computations across all silos is $RKM$:
\als
RKM\geq&\frac{\epsilon_0nM\sqrt{R}}{4\sqrt{2\ln(2/\delta_0)}}=\frac{\epsilon_0nM}{4\sqrt{2\ln(2/\delta_0)}}\sqrt{\frac{\beta D}{L}\sqrt{nM}\min\left\{1,\frac{\epsilon_0\sqrt{n}}{\sqrt{d}}\right\}+nM\min\left\{1,\frac{\epsilon_0^2n}{d}\right\}} \\
\geq&\frac{\epsilon_0(nM)^{3/2}}{4\sqrt{2\ln(2/\delta_0)}}\min\left\{1,\frac{\epsilon_0\sqrt{n}}{\sqrt{d}}\right\}
\eals
The first inequality is inputting $K$, then we input $R$, and then remove the first term under the square root.
Thus in the prevalent regime where $\epsilon_0\in\Theta(1)$ and $d\in O(n)$, then we can say that the total number of gradient computations is:
\als
RKM\geq\Omega\left((nM)^{3/2}\right)=\Omega\left(|\S|^{3/2}\right)
\eals
Thus \citep{private_fed} has computational complexity that is proportional to $|\S|^{3/2}$.

\end{document}